\documentclass[11pt]{article}

\usepackage{fullpage}
\usepackage{natbib}
\usepackage{algorithm}
\usepackage{verbatim}
\usepackage[noend]{algorithmic}
\usepackage{amsmath,amsthm,amsfonts,amssymb}
\usepackage{amsmath}
\usepackage{hyperref}
\usepackage{cleveref}
\usepackage{color}
\usepackage{mathrsfs}
\usepackage{enumitem}
\usepackage{bm}
\usepackage{multirow}
\usepackage{tikz}
\usetikzlibrary{positioning,angles,quotes}
\usepackage{booktabs}
\usepackage{makecell}
\usepackage{graphicx}
\usepackage{subfigure}
\usepackage{caption}
\usepackage{thmtools}
\usepackage{thm-restate}
\usepackage{dsfont}
\usepackage{xspace}

\newcommand{\defeq}{\mathrel{\mathop:}=}

\newcommand{\mat}[1]{\ensuremath{\mathbf{#1}}}

\newcommand{\poly}{\mathrm{poly}}

\newcommand{\E}{\mathbb{E}}
\renewcommand{\Pr}{\mathbb{P}}

\newcommand{\cO}{\mathcal{O}}
\newcommand{\tlO}{\mathcal{\tilde{O}}}

\newcommand{\N}{\mathbb{N}}
\newcommand{\R}{\mathbb{R}}

\newcommand{\I}{\mat{I}}

\newcommand{\cM}{\mathcal{M}}

\newcommand{\cD}{\mathcal{D}}

\newcommand{\KL}{\mathrm{KL}}

\newcommand{\unifsim}{\overset{\mathrm{unif}}{\sim}}
\newcommand{\Algest}{\Alg_{\mathrm{est}}}
\newcommand{\Jhat}{\widehat{J}}
\newcommand{\jhat}{\widehat{j}}

\newcommand{\vecone}{\mathbf{1}}
\newcommand{\calK}{\mathcal{K}}
\newcommand{\BigOh}[1]{\mathcal{O}\left(#1\right)}

\newcommand{\calD}{\mathcal{D}}
\renewcommand{\lnot}{\ell_0}
\newcommand{\lmax}{\ell_{\max}}

\newcommand{\Mclass}{\mathscr{M}}
\newcommand{\Pclass}{\mathscr{P}}

\newcommand{\Pembed}{\Pclass_{\mathrm{embd}}}
\newcommand{\Rembed}{\Rclass_{\mathrm{embed}}}
\newcommand{\Eembed}{\Envir_{\mathrm{embed}}}

\newcommand{\Rclass}{\mathscr{R}}

\newcommand{\Alg}{\mathsf{Alg}}
\newcommand{\Envir}{\mathscr{E}}

\newcommand{\mdp}{\mathsf{mdp}}
\newcommand{\pihat}{\widehat{\pi}}

\newcommand{\rmax}{\textsc{RMax}\xspace}
\newcommand{\zrmax}{\textsc{ZeroRMax}\xspace}

\newcommand{\EULER}{\textsc{Euler}\xspace}

\usepackage{times}
\usepackage[title]{appendix}
\allowdisplaybreaks[4]

\newtheorem{theorem}{Theorem}[section]
\newtheorem{lemma}[theorem]{Lemma}
\newtheorem{corollary}[theorem]{Corollary}

\newtheorem{proposition}[theorem]{Proposition}
\theoremstyle{definition}
\newtheorem{definition}[theorem]{Definition}

\renewcommand{\P}{\mathbb{P}}

\newcommand{\cS}{\mathcal{S}}
\newcommand{\cA}{\mathcal{A}}

\newcount\Comments  
\definecolor{darkgreen}{rgb}{0,0.5,0}
\definecolor{darkred}{rgb}{0.7,0,0}
\definecolor{teal}{rgb}{0.3,0.8,0.8}
\newcommand{\kibitz}[2]{\ifnum\Comments=1\textcolor{#1}{#2}\fi}

\begin{document}

\title{\textbf{Reward-Free Exploration for Reinforcement Learning}}

\author{Chi Jin \\ Princeton University \\ \texttt{chij@princeton.edu}
\and 
Akshay Krishnamurthy \\ Microsoft Research, New York \\ \texttt{akshay@cs.umass.edu}
\and
Max Simchowitz \\ University of California, Berkeley \\ \texttt{msimchow@berkeley.edu}
\and
Tiancheng Yu \\ Massachusetts Institute of Technology \\ \texttt{yutc@mit.edu}}

\maketitle

\newcommand{\cnote}{\textcolor[rgb]{1,0,0}{Chi: }\textcolor[rgb]{1,0,1}}

\begin{abstract}

Exploration is widely regarded as one of the most challenging aspects
of reinforcement learning (RL), with many naive approaches succumbing
to exponential sample complexity.  To isolate the challenges of
exploration, we propose a new ``reward-free RL'' framework. In the
exploration phase, the agent first collects trajectories from an MDP
$\mathcal{M}$ \emph{without} a pre-specified reward function. After
exploration, it is tasked with computing near-optimal policies under
for $\mathcal{M}$ for a collection of given reward functions.  This
framework is particularly suitable when there are many reward
functions of interest, or when the reward function is shaped by an
external agent to elicit desired behavior.

We give an efficient algorithm that conducts
$\tlO(S^2A\mathrm{poly}(H)/\epsilon^2)$ episodes of
exploration and returns $\epsilon$-suboptimal policies for
an \emph{arbitrary} number of reward functions. We achieve this by
finding exploratory policies that visit each ``significant'' state
with probability proportional to its maximum visitation probability
under \emph{any} possible policy. Moreover, our planning procedure can
be instantiated by any black-box approximate planner, such as value
iteration or natural policy gradient. We also give a
nearly-matching $\Omega(S^2AH^2/\epsilon^2)$ lower bound,
demonstrating the near-optimality of our algorithm in this setting.


\end{abstract}


\section{Introduction}

In reinforcement learning (RL), an agent repeatedly interacts with an
unknown environment with the goal of maximizing its cumulative
reward. To do so, the agent must engage in \emph{exploration},
learning to visit states in order to investigate whether they hold
high reward.

Exploration is widely regarded as the most significant challenge in
RL, because the agent may have to take precise sequences of actions to reach states with high reward. 
Here, simple randomized exploration strategies provably fail: for
example, a random walk can take exponential time to reach the corner
of the environment where the agent can accummulate high
reward~\citep{li2012sample}. While reinforcement learning has seen a
tremendous surge of recent research activity, essentially all of the
standard algorithms deployed in practice employ simple randomization
or its variants, and consequently incur extremely high sample
complexity.


On the other hand, sophisticated exploration strategies which
deliberately incentivize the agent to visit new states are provably
sample-efficient
(c.f.,~\citet{kearns2002near,brafman2002r,azar2017minimax,dann2017unifying,jin2018q}),
with recent work providing a nearly-complete theoretical understanding
for maximizing a single prespecified reward function
\cite{dann2015sample,azar2017minimax,zanette2019tighter,simchowitz2019non}. In
practice, however, reward functions are often iteratively engineered
to encourage desired behavior via trial and error (e.g.  in
constrained RL
formulations~\citep{altman1999constrained,achiam2017constrained,tessler2018reward,miryoosefi2019reinforcement}).
In such cases, repeatedly invoking the same reinforcement learning
algorithm with different reward functions can be quite sample
inefficient.  

One solution to avoid excessive data collection in such settings is to 
first collect a dataset with good coverage over all possible scenarios 
in the environment,
and then apply a  ``Batch-RL'' algorithm. Indeed many algorithms are known for computing 
near optimal policies from previously collect data, provided that the dataset 
has good coverage~\citep{munos2008finite,antos2008learning,chen2019information,agarwal2019optimality}. 
However, prior work provides little guidance into how to obtain such good coverage.

In this paper, we aim to develop an end-to-end instantiation of this proposal. To this end we ask: 




\begin{center}
    \textbf{How can we efficiently explore an environment without using any reward information?} 
\end{center}

In particular, by exploring the environment, we aim to gather sufficient information so that we can compute the near-optimal policies for \emph{any} reward function after-the-fact.

\paragraph{Our Contributions.}
In this paper, we present the first near-optimal upper and lower
bounds which characterize the sample complexity of achieving provably
sufficient coverage for Batch-RL. We do so by adopting a novel
``reward-free RL'' paradigm: During an exploration
phase, the agent collects trajectories from an MDP $\mathcal{M}$
\emph{without} a pre-specified reward function. Then, in a planning
phase, it is tasked with computing near-optimal policies under the
transitions of $\mathcal{M}$ for a large collection of given reward
functions.

Letting $S$ denote the number of states, $A$ the number of actions,
$H$ the horizon, and $\epsilon$ the desired accuracy, we give an
efficient algorithm which, after conducting
$\tlO(S^2A\mathrm{poly}(H)/\epsilon^2)$ episodes of
exploration, collects a data set with sufficiently good coverage to
enable application of standard Batch-RL solvers. Specifically, we show
that when given a reward function $r$ we can find an
$\epsilon$-suboptimal policy for the true MDP $\cM$ with reward $r$,
using the dataset alone and no additional data collection.
This
guarantee holds for all possible reward functions simultaneously,
without needing to collect more data to ensure statistical correctness
as new reward functions are considered. 

Our exploration phase is conceptually simple, using an existing RL
algorithm as a black-box~\citep{zanette2019tighter}, and our planning
phase accommodates arbitrary Batch-RL solvers. We instantiate our
result with value iteration and natural policy gradient as special
cases.  By decoupling exploration and planning, our work sheds light
on the algorithmic mechanisms required for sample efficient
reinforcement learning. We hope that this insight will be useful in
the design of provably efficient algorithms for more practically
relevant RL settings, such as those where function
approximation is required.

In addition to our algorithmic results, we establish a nearly-matching
$\Omega(S^2AH^2/\epsilon^2)$ lower bound, demonstrating the
near-optimality of our algorithm in this paradigm. Notably, this lower
bound quantifies a price of ``good-coverage'' in the reward-free
setting: while RL with a pre-specified reward has sample complexity of
only $\widetilde{\Theta}(SA H^2/\epsilon^2)$ \citep{dann2015sample},
the reward-free sample complexity is a factor of $S$ larger.

\paragraph{Technical Novelty.} 
The main technical challenge in our work involves handling
environments with states that are difficult to reach. In such cases,
we cannot learn the transition operator to high accuracy uniformly
over the environment, simply because we cannot reach these states to
collect enough data. With $\lambda (s)$ denoting the maximal probability of
visiting state $s$ under any policy, our key observation is that we
can partition the state space into two groups: the states with
$\lambda (s)$ so small that they have negligible contribution to reward
optimization, and the rest.
We introduce a rigorous analysis which enables us to ``ignores'' the
difficult-to-visit states altogether and only requires that we visit
the remaining states with probability proportional $\lambda (s)$. To
achieve this latter guarantee, we conduct our exploration with the
\EULER algorithm~\citep{zanette2019tighter}, which in our context
yields refined sample complexity guarantees in terms of $\lambda (s)$.  We
believe that this decomposition of states into their ease of being
reached may be of broader interest. Our lower bound also adopts a
novel and sophisticated construction, detailed in
Section~\ref{sec:lb}.

\paragraph{Related work.}
For reward-free exploration in the tabular setting, we are aware of
only a few prior approaches. First, when one runs a PAC-RL algorithm
like \rmax with no reward function~\citep{brafman2002r}, it does visit
the entire state space and can be shown to provide a coverage
guarantee. However, for \rmax in particular the resulting sample
complexity is quite poor, and significantly worse than our
near-optimal guarantee (See Appendix~\ref{app:rmax} for a detailed
calculation). We expect similar behavior from other PAC algorithms,
because reward-dependent exploration is typically suboptimal for the
reward-free setting.


Second, one can extract the exploration component of recent results
for RL with function
approximation~\citep{du2019provably,misra2019kinematic}.
Specifically, the former employs a model based approach where a model
is iteratively refined by planning to visit unexplored states, while
the latter uses model free dynamic programming to identify and reach
all states.  While these papers address a more difficult setting, it
is relatively straightforward to specialize their results to the
tabular setting.  In this case, both methods guarantee coverage, but
they have suboptimal sample complexity and require that all states can
be visited with significant probability. In contrast, our approach
requires no visitation probability assumptions and achieves the
optimal sample complexity.

The last point of comparison is a recent result
of~\citet{hazan2018provably}, that gives an efficient algorithm for
finding a certain exploratory policy. They use a Frank-Wolfe style
algorithm to find a policy whose state occupancy measure has maximum
entropy.  One can show that an exact optimizer for their objective has
a similar coverage property to our exploratory policy, but the
Frank-Wolfe style algorithm can only guarantee an approximate
optimizer. They do not analyze how the optimization error enters in
the coverage guarantee, but we are able to show that setting the error
to $O(1/S)$ suffices (see Appendix~\ref{app:max_ent}). Unfortunately,
this implies that their sample complexity scales with $S^5$, which is
much worse than ours. More generally, their result is not end-to-end
in that they do not show how to use their policy for planning, and they
do not establish a final sample complexity bound, both of which we do
here.

Finally, the main source of motivation for our work is recent and
classical results on batch reinforcement
learning~\citep{munos2008finite,antos2008learning,chen2019information,agarwal2019optimality},
a setting where the goal is to find a near optimal policy, given
an \emph{a priori} dataset collected by some logging policy that
satisfies certain coverage properties. In this paper, we show how to
find such a logging policy for the tabular setting, which enables
straightforward application of these batch RL results. As an example,
we show how to apply both value iteration and natural policy gradient
to optimize the policy given any reward function. More generally, these works typically
also consider the function approximation setting, and we believe our
modular approach will facilitate development of provably efficient
algorithms for these challenging settings.



\section{Preliminaries}
\label{prelim}
We consider the setting of a tabular episodic Markov decision process,
$\rm{MDP}(\cS, \cA, H, \P, r)$, where $\cS$ is the set of states with $|\cS| = S$,
$\cA$ is the set of actions with $|\cA| = A$, $H$ is the number of steps in each episode,
$\P$ is the time-dependent transition matrix so that $\P_h ( \cdot | s, a) $ gives the distribution
over the next state if action $a$ is taken from state $s$ at step $h\in [H]$, and $r_h \colon \cS \times \cA \to [0,1]$ is the deterministic reward function at step $h$.%
\footnote{While we study deterministic reward functions for notational simplicity, our results generalize to randomized reward functions. }
Note that we are assuming that rewards are in $[0,1]$ for normalization.

In each episode of a standard MDP, an initial state $s_1$ is picked
from an unknown initial distribution $\P_1(\cdot)$. Then, at each step
$h \in [H]$, the agent observes state $s_h \in \cS$, picks an action
$a_h \in \cA$, receives reward $r_h(s_h, a_h)$, and then transitions
to the next state $s_{h+1}$, which is drawn from the distribution
$\P_h(\cdot | s_h, a_h)$. The episode ends after the $H^{\textrm{th}}$
  reward is collected.

A (non-stationary, stochastic) policy $\pi$ is a collection of $H$
functions $\big\{ \pi_h: \cS \rightarrow \Delta_\cA \big\}_{h\in
  [H]}$, where $\Delta_\cA$ is the probability simplex over action set
$\cA$. As notation, we use $\pi(\cdot|s)$ to denote the action
distribution for policy $\pi$ in state $s$.  We use
$V^\pi_h \colon \cS \to \mathbb{R}$ to denote the value function at
step $h$ under policy $\pi$, which 
gives the expected
sum of remaining rewards received under policy $\pi$, starting from
$s_h = s$, until the end of the episode. That is,
\begin{equation*}
V^\pi_h(s) \defeq \E_{\pi}\left[\sum_{h' = h}^H r_{h'}(s_{h'}, a_{h'}) | s_h = s\right] .
\end{equation*}
Accordingly, we also define $Q^\pi_h:\cS \times \cA \to \mathbb{R}$ to
denote action-value function at step $h$, so that $Q^\pi_h(s, a)$
gives the expected sum of remaining rewards received under policy
$\pi$, starting from $s_h = s, a_h = a$, until the end of the
episode. Formally:
\begin{equation*}
Q^\pi_h(s,a) \defeq \E_{\pi}\left[\sum_{h' = h}^H r_{h'}(s_{h'}, a_{h'}) | s_h = s, a_h = a\right] .
\end{equation*}

Since the state and action spaces, and the horizon, are all finite,
there always exists (see, e.g., \cite{azar2017minimax}) an optimal
policy $\pi^\star$ which gives the optimal value $V^\star_h(s) =
\sup_{\pi} V_h^\pi(s)$ for all $s\in \cS$ and $h\in [H]$.  As
notation, define $[\P_h V_{h+1}](s, a) \defeq \E_{s' \sim
  \P(\cdot|s, a)} V_{h+1}(s')$. Recall the Bellman equation
\begin{align}\label{eq:bellman}
V^\pi_h(s) = Q^\pi_h(s, \pi_h(s)), \qquad  Q^\pi_h(s, a) = (r_h + \P_h V^\pi_{h+1})(s, a)
\end{align}
and the Bellman optimality equation:
\begin{align}\label{eq:bellman_opt}
V^\star_h(s) = \max_{a\in\cA}Q^\star_h(s, a), \qquad Q^\star_h(s, a) \defeq (r_h + \P_h V^\star_{h+1})(s, a) .
\end{align}
where we define $V^\pi_{H+1}(s) = V^\star_{H+1}(s) = 0$ for any $s \in \cS$.

The RL objective is to find an $\epsilon$-optimal policy $\pi$, satisfying
\begin{align*}
\E_{s_1 \sim \P_1} [V_1^{\star}(s_1) - V_1^{\pi}(s_1)] \le \epsilon
\end{align*}

\makeatletter
\renewcommand{\ALG@name}{Protocol}
\makeatother

\begin{algorithm}[tb]
   \caption{Reward-Free Exploration}
   \label{pro:rfMDP}
\begin{algorithmic}
   \FOR{$k=1$ {\bfseries to} $K$}
   \STATE learner decides a policy $\pi_k$
   \STATE environment samples the initial state $s_0 \sim \P_1$.
   \FOR{$h=1$ {\bfseries to} $H$}
   \STATE learner selects action $a_h \sim \pi_h(\cdot|s_h)$
   \STATE environment transitions to $s_{h+1} \sim \P_h(\cdot|s_h,a_h)$
   \STATE learner observes the next state $s_{h+1}$
   \ENDFOR
   \ENDFOR
\end{algorithmic}
\end{algorithm}

\makeatletter
\renewcommand{\ALG@name}{Algorithm}
\makeatother

\paragraph{Reward-free Exploration.}
In the reward-free setting, we would like to design algorithms that
efficiently explore the state space without the guidance of reward
information. Formally, the agent interacts with the environment
through Protocol \ref{pro:rfMDP}---a reward-free version of the MDP,
where the agent can transit as usual but does not collect any
rewards. Over the course of $K$ episodes following
Protocol~\ref{pro:rfMDP}, the agent collects a dataset of visisted
states, actions, and transitions
$\mathcal{D} = \{s^{(k)}_h, a^{(k)}_h\}_{(k, h) \in [K]\times[H]}$,
which is the outcome of the \emph{exploration phase}.

The effectiveness of the exploration strategy is evaluated in the next
phase---the \emph{planning phase}---in which the agent is no longer
allowed to interact with the MDP. In this phase, the agent is given a
reward function $r(\cdot, \cdot)$ that can be potentially adversarily
designed, and the \textbf{objective} here is to compute a near optimal policy for this reward
function using the dataset $\mathcal{D}$. 
Performance is measured in terms of how many episodes $K$ are
required in the exploration phase so that the agent can reliably
achieve the objective above.  As notation, we use $V(\cdot;r)$ to
emphasize that the value function depends on the reward $r$. 

We remark that providing the reward function after the exploration
phase (as opposed to before) makes the setting more challenging, and
so our algorithm applies to the easier setting. We also note that
our results address the setting where the reward is observed through
interaction with the environment, as learning the reward is typically
not the statistical barrier to efficient RL. Indeed, a provably effective reward-free exploration
strategy must visit all ``significant'' state-action pairs (see
Definition~\ref{def:significant}) sufficiently many times anyway, and
this experience is sufficient to learn the reward function.

\section{Main Results} \label{sec:main_results}
We are now ready to state our main theorem. It asserts that our
algorithm, which we will describe in the subsequent sections, is a
reward-free exploration algorithm with sample complexity
$\tlO(H^5S^2A/\epsilon ^2)$, ignoring lower order terms. In other
words, after this many episodes interacting with the MDP via
Protocol~\ref{pro:rfMDP}, our algorithm can compute $\epsilon$-optimal
policies for arbitrarily many reward functions.  The theorem
demonstrates that the sample complexity of reward-free exploration is
at most $\tlO(H^5S^2A/\epsilon ^2)$, which we will show to be
near-optimal with our lower bound in the next section.


\begin{theorem}
  \label{thm:main}
Ther exists an absolute constant $c>0$ and a reward-free exploration
algorithm such that, for any $p \in (0,1)$, with probability at least
$1-p$, the algorithm outputs $\epsilon$-optimal policies for an
arbitrary number of adaptively chosen reward functions. The number of
episodes collected in the exploration phase is bounded by
\begin{equation}
  \label{equ:main}
  c\cdot \left[ \frac{H^5S^2A\iota}{\epsilon ^2}+\frac{S^{4}AH^7 \iota ^3}{\epsilon} \right],
\end{equation}
where $\iota \defeq \log(SAH/(p\epsilon))$.
\end{theorem}

We emphasize that the correctness guarantee here is quite strong: the
dataset $\mathcal{D}$ collected by the algorithm is such that any
number of adaptively chosen reward functions can be optimized with no
further data collection. In contrast, if we na\"{i}vely deployed a
reward-sensitive RL algorithm, we would have to collect additional
trajectories for each reward function, which could be quite sample
inefficient. We emphasize that requiring near-optimal policies for
many reward functions is quite common in applications, especially when
we design reward functions by trial and error to elicit specific
behaviors.

\paragraph{Algorithm overview.}
Our algorithm proceeds with following high level steps:
\begin{enumerate}
\item learn a set of policies $\Psi$ which allow us to visit all
  ``significant'' states with reasonable probability.
\item collect a sufficient amount of data by executing policies in $\Psi$.
\item compute the empirical transition matrix $\hat{\P}$ using the collected data.
\item for each reward function $r$, find a near-optimal policy by invoking a planning algorithm with transitions $\hat{\P}$ and reward $r$.
\end{enumerate}
The first two steps are performed in the exploration phase, while the latter two steps are performed in the planning phase.
In Section \ref{sec:main_exp} and Section \ref{sec:main_plan}, we will present our formal algorithms and the corresponding theoretical guarantees for two phases separately. 
One important feature of our algorithm is that we can use existing approximate MDP solvers or batch-RL algorithms in the last step.
We demonstrate with two examples, namely Value Iteration (VI) and Natural Policy Gradient (NPG), in Section \ref{sec:main_solvers}.




\subsection{Exploration Phase} \label{sec:main_exp}
\begin{algorithm}[tb]
   \caption{Reward-free RL-Explore}
   \label{alg:main_exp}
\begin{algorithmic}[1]
  \STATE {\bfseries Input:} iteration number $N_0$, $N$.
  \STATE set policy class $\Psi \leftarrow \emptyset$, and dataset $\cD \leftarrow \emptyset$.
  \FOR{all $(s, h) \in \cS \times [H]$} \label{line:policy_cover_start}
  \STATE $r_{h'}(s', a') \leftarrow \mathds{1}[s'=s \text{~and~} h'=h]$ for all $(s', a', h') \in \cS \times \cA \times [H]$.\label{line:reward_def}
  \STATE $\Phi^{(s,h)} \leftarrow \EULER(r, N_0)$.
  \STATE $\pi_{h}(\cdot|s) \leftarrow \text{Uniform}(\cA)$ for all $\pi \in \Phi^{(s,h)}$.
  \STATE $\Psi \leftarrow \Psi \cup \Phi^{(s,h)}$.
  \ENDFOR \label{line:policy_cover_end}
  \FOR{$n=1 \ldots N$}\label{line:sample_start}
  \STATE sample policy $\pi \sim \text{Uniform}(\Psi)$.
  \STATE play $\mathcal{M}$ using policy $\pi$, and observe the trajectory $z_n = (s_1, a_1, \ldots, s_H, a_H, s_{H+1})$.
  \STATE $\cD \leftarrow \cD \cup \{z_n\}$
  \ENDFOR \label{line:sample_end}
  \STATE {\bfseries Return:} dataset $\cD$.
\end{algorithmic}
\end{algorithm}

The goal of exploration is to visit all possible states so that the agent can gather sufficient information in order to find the optimal policy eventually. However, rather different from the bandit setting where agent can select an arbitrary arm to pull, it is possible that certain state in the MDP is very difficult to reach no matter what policy the agent is taking. Therefore, we first introduce the concept of the state being ``significant''. See Figure \ref{fig:mdp_significant} for illustrations.

\begin{definition} \label{def:significant}
A state $s$ in step $h$ is \textbf{$\delta$-significant} if there exists a policy $\pi$, so that the probability to reach $s$ following policy $\pi$ is greater than $\delta$. In symbol:
$$
\max_{\pi} P_{h}^{\pi}\left( s \right) \ge \delta
$$
\end{definition}

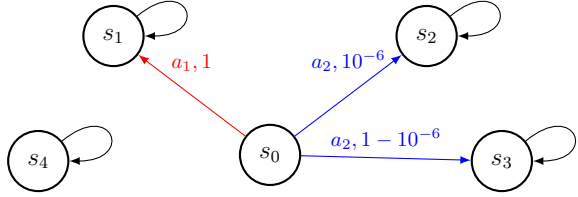
\begin{figure}[t]
  \centering 

 \begin{tikzpicture}[auto,node distance=8mm,>=latex, font = \small,scale=0.8, every node/.style={scale=0.8}]

    \tikzstyle{round}=[thick,draw=black,circle]

    \node[round,minimum size=10mm] (s0) {\large$s_0$};
    \node[round,,minimum size=10mm, above left=10mm and 15mm of s0,font = \small] (s1) {\large$s_1$};
    \node[round,,minimum size=10mm, above right=10mm and 15mm of s0,font = \small] (s2) {\large$s_2$};
    \node[round,,minimum size=10mm, below right=-5mm and 25mm of s0,font = \small] (s3) {\large$s_3$};
    \node[round,,minimum size=10mm, below left=-5mm and 25mm of s0,font = \small] (s4) {\large$s_4$};


    \draw[->,color = blue] (s0) -- node[above=2mm] {$a_2, 10^{-6}$}  (s2);
    \draw[->,color = blue] (s0) -- node[above=0mm] {$a_2, 1-10^{-6}$}  (s3);
    \draw[->,color = red] (s0) -- node[above=2mm] {$a_1, 1$}  (s1);
    
    \draw[->] (s1) [out=40,in=0,loop] to coordinate[pos=0.1](aa) (s1);
    \draw[->] (s2) [out=40,in=0,loop] to coordinate[pos=0.1](aa) (s2);
    \draw[->] (s3) [out=40,in=0,loop] to coordinate[pos=0.1](aa) (s3);
    \draw[->] (s4) [out=40,in=0,loop] to coordinate[pos=0.1](aa) (s4);

\end{tikzpicture}
  \caption{Illustration of significant states (Definition~\ref{def:significant}) v.s. insignificant states. In this toy example we have 5 states, where $s_0$ is the initial state. Only from state $s_0$ the agent can transit to other states and the other states are absorbing whatever action the agent takes. For state $s_0$, we use blue arrows to represent transition if action $a_1$ is taken and red ones if action $a_2$ is taken. The numbers on the arrows following the actions are the transition probability. In this example, $s_4$ is insignificant, because it can never be reached. For $\delta = 10^{-5}$, $s_2$ is also $\delta$-insignificant, because the best policy to reach $s_2$ is by taking action $a_2$ at initial state $s_0$, which gives the maximum probability $10^{-6}$ to reach $s_2$. The remaining states $s_1, s_3$ are all $\delta$-significant.
  }\label{fig:mdp_significant}
\end{figure}

Intuitively, with limited budeget of samples and runtime, one can be only hopefully to visit all significant states. On the other hand, since insignificant states can be rarely visited no matter what policy is used, they will not significantly change the value from the initial states. Thus, for the sake of finding near-optimal policies, it is sufficient to visit all significant states with proper significance level $\epsilon$. 
Indeed, Algorithm \ref{alg:main_exp} is able to provide such a guarantee as follows.

\begin{theorem}
  \label{thm:main_exp}
  There exists absolute constant $c>0$ such that for any $\epsilon >0$ and $p \in (0, 1)$, if we set $N_0\ge cS^{2}AH^4 \iota_0 ^3/\delta$ where $\iota_0 \defeq \log(SAH/(p\delta))$, then with probability at least $1-p$, that Algorithm \ref{alg:main_exp}
  will returns a dataset $\mathcal{D}$ consisting of $N$ trajectories $\{z_n\}_{n=1}^N$, which are i.i.d sampled from a distribution $\mu$ satisfying:
  \begin{equation} \label{eq:d_cond}
  \forall \text{~$\delta$-significant~} (s,h), \quad \max_{a, \pi}\frac{P_{h}^{\pi}(s,a)}{\mu_h(s,a)} \le 2SAH.
  \end{equation}
\end{theorem}

Theorem \ref{thm:main_exp} claims that using Algorithm \ref{alg:main_exp}, we can collect data from a underlying distribution $\mu$, which ensures that for policy $\pi$, the ratio $P_{h}^{\pi}(s,a)/\mu_h(s,a)$ will be upper bounded for any significant state and action. That is, all significant state and action will be visited by distribution $\mu$ with reasonable amount of probability. Notice as $\delta$ becomes smaller, there will be more significant states and the condition~\eqref{eq:d_cond} becomes stronger. As a result we need to take larger $N_0$. As we will see later, the $\delta$ we take eventually will be $\epsilon /\left( 2SH^2 \right)$, where $\epsilon$ is the suboptimality of the policy we find in the planning phase.

Algorithm \ref{alg:main_exp} can be decompose into two parts, where Line \ref{line:policy_cover_start}-\ref{line:policy_cover_end} learns a set of exploration policies $\Psi$ and Line \ref{line:sample_start}-\ref{line:sample_end} simply collects data by uniformly executing policies in $\Psi$. Therefore, the key mechanism lies in how to learn the set of exploration policies $\Psi$. Our strategy is to first learn the best policies that maximize the probability to research each state $s$ at step $h$ individually, and then combine them.

Concretely, for each state $s$ at step $h$, algorithm \ref{alg:main_exp} first create a reward function $r$ that is always zero except for the state $s$ at step $h$. Then we can simulate a standard MDP by properly feeding this designed reward $r$ when an agent interact with the environment using protocol \ref{pro:rfMDP}. It is easy to verify that the optimal policy for the MDP with this reward $r$ is precisely the policy that maximizes the probability to reach $(s, h)$. Thus, any RL algorithms with PAC or regret guarantees \cite{azar2017minimax,jin2018q} can be used here to approximately find this optimal policy. In particular, we use \EULER algorithm \cite{zanette2019tighter}, whose theoretical guarantee in our setting is presented as follows \footnote{In \cite{zanette2019tighter}, \EULER is studied under stationary setting, where $\P$ and $r$ does not depend on $h$. A stationary MDP can simulate a non-stationary MDP by augmenting state $s$ to $(s, h)$. Therefore, the effective number of states becomes $SH$ when we apply the results in \cite{zanette2019tighter}.}

\begin{lemma}
  \label{zanette}
There exists absolute constant $c>0$ such that for any $N_0>0$ and $p \in (0, 1)$, with probability at least $1-p$, if we run \EULER algorithm for $N_0$ episodes, it will output a policy set $\Phi$ with $|\Phi| = N_0$ that satisfies: 
\begin{align*}
  &\E_{s_1\sim \P _1}\left[ V_{1}^{\star}\left( s_1 \right) -\frac{1}{N_0}\sum_{\pi \in \Phi}{V_{1}^{\pi}}\left( s_1 \right) \right] \le c\cdot \left\{ \sqrt{\frac{SAH\iota_0 \cdot \E_{s_1\sim \P _1}V_{1}^{\star}\left( s_1 \right)}{N_0}}+\frac{S^2AH^{4}\iota_0 ^3}{N_0} \right\} 
\end{align*}
where $\iota _0=\log \left( SAHN_0/p \right)$.
\end{lemma}

We comment that one unique feature of \EULER algorithm is that its suboptimality scales with the value of the optimal policy $\E_{s_1 \sim \P_1} V_1^{\star}(s_1)$. This is key in obtaining a sharp result, and is especially helpful in dealing with those states that are still significant but their maximum reaching probability is low. Finally, since the best policy to reach $(s, h)$ is only meaningful at steps before $h$, algorithm \ref{alg:main_exp} then alter the policy for state $s$ at step $h$ to be $\text{Uniform}(\cA)$ to ensure good probability of choosing all actions for this state.

\subsection{Planning Phase} \label{sec:main_plan}

\begin{algorithm}[tb]
   \caption{Reward-free RL-Plan}
   \label{alg:main_plan}
\begin{algorithmic}[1]
  \STATE {\bfseries Input:} a dataset of transition $\cD$, reward function $r$, accuracy $\epsilon$.
  \FOR{all $(s, a, s', h) \in \cS \times \cA \times \cS \times [H]$} \label{line:empricial_start}
  \STATE $N_h(s, a, s') \leftarrow \sum_{(s_h, a_h, s_{h+1}) \in \cD} \mathds{1}[s_h = s, a_h = a, s_{h+1} = s'] $.
  \STATE $N_h(s, a) \leftarrow \sum_{s'} N_h(s, a, s') $.
  \STATE $\hat{\P}_h(s'|s, a) = N_h(s,a, s')/N_h(s, a)$.
  \ENDFOR\label{line:empirical_end}
  \STATE $\hat{\pi} \leftarrow \textsc{APPROXIMATE-MDP-SOLVER}(\hat{\P}, r, \epsilon)$. 
  \STATE {\bfseries Return:} policy $\hat{\pi}$.
\end{algorithmic}
\end{algorithm}

In planning phase, the agent is given the reward function $r$, and aims to find a near-optimal policy based on $r$ and dataset $\cD$ collected in the exploration phase. Algorithm \ref{alg:main_plan} proceeds with two steps. Line \ref{line:empricial_start}-\ref{line:empirical_end} use counts based on dataset $\cD$ to estimate the empirical transition matrix $\hat{\P}$. Then, algorithm \ref{alg:main_plan} calls a approximate MDP solver. 
Subroutine \textsc{APPROXIMATE-MDP-SOLVER}$(\hat{\P}, r, \epsilon)$ can be any algorithm that finds $\epsilon$-suboptimal policy $\hat{\pi}$ for MDP with known transition matrix and reward (they are $\hat{\P}$, $r$ in this case). See Section \ref{sec:main_solvers} for examples of such approximate MDP solvers.

Now we are ready to state the guarantee for Algorithm \ref{alg:main_plan}, which asserts that as long as the number of data collected in the exploration phase is sufficiently large, the output policy $\hat{\pi}$ is not only a near-optimal policy for the estimated MDP with transition $\hat{\P}$, but also a near-optimal policy for the true MDP.


\begin{theorem}\label{thm:plan}
There exists absolute constant $c>0$, for any $\epsilon>0$, $p\in(0, 1)$, assume dataset $\cD$ has $N$ i.i.d. samples from distribution $\mu$ which satisfies Eq.\eqref{eq:d_cond} with $\delta=\epsilon /\left( 2SH^2 \right) $, and $N \ge cH^{5}S^2A\iota/\epsilon^2$, then with probability at least $1-p$, for any reward function $r$ simultanouesly, the output policy $\hat{\pi}$ of Algorithm \ref{alg:main_plan} is $3\epsilon$-suboptimal. That is:
\begin{equation*}
\E_{s_1 \sim \P_1} [V_1^{\star}(s_1; r) - V_1^{\hat{\pi}}(s_1; r)] \le 3\epsilon
\end{equation*}
\end{theorem}

The mechanism behind Theorem \ref{thm:plan} is that: by sample sufficient number of exploring data, we ensure that the empirical transition $\hat{\P}$ and the true transition $\P$ are close so that the near-optimal policy for the esimated MDP with transition $\hat{\P}$ is also near optimal for the true MDP. We note that the closeness of $\hat{\P}$ and $\P$ can not be established in the usual sense of the TV-distance (or other distributional distance) between $\hat{\P}_h(\cdot|s, a)$ and $\P_h(\cdot|s, a)$ is small for any $(s, a, h)$, due to the existence of insignificant states. The key observation is that, nevertheless, we can establish the closeness of $\hat{\P}$ and $\P$ in the sense that for any policy $\pi$, the value functions starting from initial states are close. That is, the difference in policy evaluations of two MDPs is small, which is summarized in the following lemma.

\begin{lemma}
  \label{lem:plan}
Under the preconditions of Theorem \ref{thm:plan},  with probability at least $1-p$, for any reward function $r$ and any policy $\pi$, we have:
\begin{equation}
|\E_{s_1 \sim \P_1} [\hat{V}^{\pi}_{1}(s_1; r) -V^{\pi}_{1}(s_1; r)]| \le \epsilon
\end{equation}
where $\hat{V}$ is the value function of MDP with the transition $\hat{\P}$.
\end{lemma}

The establishment of Lemma \ref{lem:plan} is a natual consequence of the followings: (1) the total contribution from all insignificant states is small; (2) $\hat{\P}$ is reasonably accurate for all significant states; and (3) a new sharp concentration inequality (see Lemma \ref{lem:sharp_concentration} in Appendix). With Lemma \ref{lem:plan}, now we are ready to prove Theorem \ref{thm:plan}.



\begin{proof}[Proof of Theorem~\ref{thm:plan}]
We denote the optimal policy of MDP$(\P, r)$ and MDP$(\hat{\P}, r)$ by $\pi^\star$ and $\hat{\pi}^\star$ respectively.
The theorem is a direct consequence of the following decomposition
\begin{align*}
  &\E_{s_1\sim \P _1}\{ V_{1}^{\pi ^{\star}}( s_1;r ) -V_{1}^{\hat{\pi}}( s_1;r ) \} 
\\
\le&\underset{\text{Evaluation error }1}{\underbrace{| \E_{s_1\sim \P _1}\{ V_{1}^{\pi ^{\star}}( s_1;r ) -\hat{V}_{1}^{\pi ^{\star}}( s_1;r ) \} |}}+\underset{\le \text{0 by definition}}{\underbrace{\E_{s_1\sim \P _1}\{ \hat{V}_{1}^{\pi ^{\star}}( s_1;r ) -\hat{V}_{1}^{\hat{\pi}^{\star}}( s_1;r ) \} }}
\\
+&\underset{\text{Optimization error}}{\underbrace{\E_{s_1\sim \P _1}\{ \hat{V}_{1}^{\hat{\pi}^{\star}}( s_1;r ) -\hat{V}_{1}^{\hat{\pi}}( s_1;r ) \} }}+\underset{\text{Evaluation error }2}{\underbrace{| \E_{s_1\sim \P _1}\{ \hat{V}_{1}^{\hat{\pi}}( s_1;r ) -V_{1}^{\hat{\pi}}( s_1;r ) \} |}}
\end{align*}
where evaluation errors are bounded by $\epsilon$ by Lemma~\ref{lem:plan} and optimization error is bounded by $\epsilon$ by assumption.
\end{proof}

\subsection{Approximate MDP Solvers} \label{sec:main_solvers}

\begin{algorithm}[tb]
   \caption{Natural Policy Gradient (NPG)}
   \label{alg:NPG}
\begin{algorithmic}[1]
  \STATE {\bfseries Input:} transition matrix $\P$, reward function $r$, stepsize $\eta$, iteration number $T$.
  \STATE initialize $\pi_h^{(0)} (\cdot|s) \leftarrow \text{Uniform}(\cA)$ for all $(s, h)$
  \FOR{$t = 0, \cdots, T-1$}
  \STATE evaluate $Q^{\pi^{(t)}}_h(s, a)$ using Bellman equation Eq.\eqref{eq:bellman} for all $(s, a, h)$.
  \STATE update $\pi^{(t+1)}_h(a |s) \propto \pi^{(t)}_h(a |s) \cdot \exp(\eta Q^{\pi^{(t)}}_h(s, a))$ for all $(s, a, h)$.
  \ENDFOR
  \STATE {\bfseries Return:} policy $\pi^{(T)}$.
\end{algorithmic}
\end{algorithm}

Approximate MDP solvers aim to find a near-optimal policy when the exact transition matrix $\P$ and reward $r$ are known. The simplest way to achieve this is by \textbf{Value Iteration} (VI) algorithm, which solves the Bellman optimality equation Eq.\eqref{eq:bellman_opt} in a dynamical programming fashion. Then the greedy policy induced by the result $Q^\star$ gives precisely the optimal policy without error.

Another popular approach frequently used in practice is the \textbf{Natural Policy Gradient} (NPG) algorithm as shown in Algorithm \ref{alg:NPG}. In each iteration, the algorithm first evaluates the value of policy $\pi^{(t)}$ using Bellman equation Eq.\eqref{eq:bellman}. Then it updates the policy by first scale it with the exponential of learning $\eta$ times value $Q^{\pi^{(t)}}$, and then performs a normalization. For completeness, we provides its guarantee here. Similar analysis also appears in \cite{agarwal2019optimality}.

\begin{proposition}
  \label{prop:NPG}
  for any learning rate $\eta$ and iteration number $T$, the output policy $\pi^{(T)}$ of Algorithm \ref{alg:NPG} satisfies the following:
$$ \E_{s_1 \sim \P_1} [V_{1}^\star(s_1)-V_{1}^{\pi^{(T)}}(s_1)]\le \frac{H\log A}{\eta T}+\eta H^2 $$
\end{proposition}

Therefore, it is easy to verify, by choosing $\eta = \sqrt{\log A /HT}$ and $T = 4H^3 \log A /\epsilon^2$, the policy $\pi^{(T)}$ returned by NPG is $\epsilon$-optimal.


\section{Lower Bound\label{sec:lb}}
\newcommand{\calX}{\mathcal{X}}
\newcommand{\calM}{\mathcal{M}}
\newcommand{\Exp}{\mathbf{E}}

\newcommand{\gtrsimst}{\gtrsim_{\star}}
\newcommand{\lesssimst}{\lesssim_{\star}}

In this section, we establish that $\Omega(H^2 S^2A/\epsilon^2)$
trajectories are necessary to satisfy the guarantee from Theorem~\ref{thm:main}.

\begin{theorem}\label{thm:main_lb} Let $C > 0$ be a universal constant. Then for $A \ge 2$, $S \ge C \log_2 A$, $H \ge C\log_2
S$, and any $\epsilon \le \min\{1/4,H/48\}$, any
reward-free exploration algorithm $\Alg$ which statisfies the guarantee of Theorem~\ref{thm:main} with $p = 1/2$ and accuracy parameter $\epsilon$ must collect $\Omega(S^2 A
H^2/\epsilon^2)$ trajectories in expectation. This is true even if
$\Alg$ can return randomized or history-dependent (non-Markov)
policies, and holds even if the  rewards and transitions are identical across stages $h$.
\end{theorem}
In particular, Theorem~\ref{thm:main_lb} shows that our upper bound (Theorem \ref{thm:main}) is tight in $S, A, \epsilon$, up to logarithmic factors
and lower-order terms. Note that lower bound holds against querying an
\emph{unlimited} number of reward vectors. It is left as an open
question whether such a lower bound holds when the algorithm is only
required to ensure correctness over a smaller number of reward vectors
pre-determined in advance. In what follows, we sketch a proof of
Theorem~\ref{thm:main_lb}; a formal proof is given in
Appendix~\ref{app:proof_of_lower_bound}.

\subsection{Reward Free Exploration at a Single State}

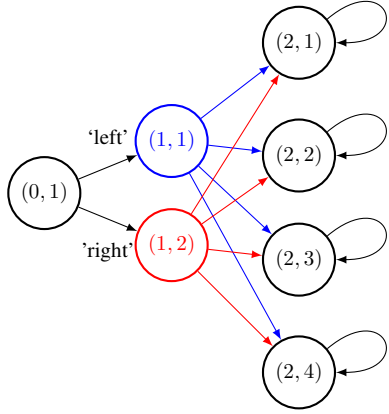
\begin{figure}[t!]
    \centering
    \begin{subfigure}
        \centering

  \begin{tikzpicture}[auto,node distance=8mm,>=latex, font = \small,scale=0.8, every node/.style={scale=0.8}]

    \tikzstyle{round}=[thick,draw=black,circle]

    \node[round] (11) {$(0,1)$};
    \node[round,above right=0mm and 10mm of 11,font = \small,color = blue] (21) {$(1,1)$};
    \node[round,below right=0mm and 10mm of 11,font = \small,color = red] (22) {$(1,2)$};
    \node[round,above right=20mm and 10mm of 22,font = \small] (31) {$(2,1)$};
    \node[round,above right=5mm and 10mm of 22,font = \small] (32) {$(2,2)$};
     \node[round,below right=-5mm and 10mm of 22,font = \small] (33) {$(2,3)$};
     \node[round,below right=10mm and 10mm of 22,font = \small] (34) {$(2,4)$};


    \draw[->] (11) -- node[above=2mm] {`left'}  (21);
    \draw[->] (11) -- node[below=2mm]{'right'} (22);

    \draw[->,color = blue] (21) -- (31);
 	\draw[->,color = blue] (21) -- (32);
    \draw[->,color = blue] (21) -- (33);
    \draw[->,color = blue] (21) -- (34);

    \draw[->,color = red] (22) -- (31);
 	\draw[->,color = red] (22) -- (32);
    \draw[->,color = red] (22) -- (33);
    \draw[->,color = red] (22) -- (34);

    \draw[->] (31) [out=40,in=0,loop] to coordinate[pos=0.1](aa) (31);
    \draw[->] (32) [out=40,in=0,loop] to coordinate[pos=0.1](aa) (32);
    \draw[->] (33) [out=40,in=0,loop] to coordinate[pos=0.1](aa) (33);
    \draw[->] (34) [out=40,in=0,loop] to coordinate[pos=0.1](aa) (34);
\end{tikzpicture}
       \caption{ The ``left'' (blue) instance and ``right'' (red) instance embed two copies of the instance from Lemma~\ref{lem:informal_single_state_lb}. In each copy, the agent begins in stage $s=0$, and moves to states $s \in [2n]$, $n = 2$. Different actions correspond to different probability distributions over next states $s \in [2n]$. States $s \in [2n]$ are absording, and rewards are action-independent. Lemma~\ref{lem:informal_single_state_lb} shows that this construction requires the learner to learn $\Omega(n)$ bits about the transition probabilities $p(\cdot|0,a)$. By embedding this coonstruction into a large MDP, this construction forces the learner to learn the transition probabilities at $n = 2$ states, $\{(x, \log_2 n) : x \in [n]\}$. The learner can determinsitically access these states by appropriate choice of ``left'' and ``right'' actions. \label{fig:mdp_embed}}
    \end{subfigure}
\end{figure}

The core of our construction is a simple instance with a single
initial state $x_1 = 0$ and $2n$ absorbing states $s \in [2n]$; the
transition from states $0 \to s$ is described by a vector $q \in \R^{[2n]\times [A]}$, where $q(s, a)$ is the transition probability to state $s$ if action $a$ is taken at state $0$. 
We shall also restrict to vectors $q$ are
close to uniform, i.e.,
\begin{align}
\forall s,a, \quad \left|q(s,a) - \frac{1}{2n}\right| \le \frac{\epsilon}{2n} \label{eq:q_condition}
\end{align}

The learner is then tasked with learning near optimal policies for
reward vectors $r_{\nu}$ parametrized by $\nu \in [0,1]^{2n}$, which
assigns a state-dependent but action-independent reward $\nu(s)$ to
states $s \in [2n]$, and no reward to $x_1 = 0$. The  blue (``left'') transitions or red (``right'') transition in Figure~\ref{fig:mdp_embed} mirror this construction, which we
formalize in Definition~\ref{defn:hard}. We show that reward-free
exploration essentially forces the learner to learn the probability vectors
$q(\cdot,a)$ in total-variation distance for each $a
\in [A]$, yielding an $\Omega(nA/\epsilon^2)$ lower bound for this
construction.  A formal statement is of the following Lemma is given in Lemma~\ref{lem:single_instance_formal} in the appendix.

\begin{lemma}[Informal]\label{lem:informal_single_state_lb}
Suppose $S \ge C\log_2(A)$ for a universal constant $C > 0$. Suppose $\Alg$, when faced with the
instances described above (with $q$ satisfying
Eq.~\eqref{eq:q_condition}) successfully returns $\epsilon$-suboptimal
policies for exponentially many reward vectors with total failure
probability $1/2$. Then $\Alg$ requires $\Omega(SA/\epsilon^2)$
trajectories in expectation.
\end{lemma}

\begin{proof}[Proof Sketch]
Unfortunately, we cannot show a direct reduction from estimating $q$
in total variation to learning near optimal-policies. Instead, by
selecting appropriate reward vectors $r_\nu$, the algorithm can decode a
packing of $\exp(\Omega(n))$ transition vectors $q(\cdot,a)$ for each
action $a \in [A]$. By a variant of Fano's inequality, this leads to
the same $\Omega(nA/\epsilon^2)$ lower bound that would be obtained by
a direct reduction.
\end{proof}
Lemma~\ref{lem:informal_single_state_lb} differs from existing
$\Omega(SA/\epsilon^2)$ lower bounds in that the only quantities
unknown to the learner are the transition probabilities associated
with the single state $0$. This is in contrast to most existing lower
bounds where the learner needs to collect transition information at
multiple states. In particular, here the factor of $S$ arises because
the transition is \emph{to} $\Theta(S)$ states, while in most
constructions this factor arises because transitions \emph{from}
$\Theta(S)$ states must be estimated.


\subsection{Lower Bound for Multiple States}
To obtain an $\Omega(S^2 H^2A/\epsilon^2)$ lower bound, we embed $n=
\Omega(S)$ instances from above as the second-to-last layer of a
binary tree of depth $1 + \log_2 n$. All $n$ such instances share the
same $2n$-terminal leaves (assume $n$ is a power of $2$). We index
states by pairs $(x,\ell)$, where $\ell$ denotes the layer. From the
binary tree construction, there are at most $4n$ states, so $n =
\Omega(S)$. We assume that the MDP begins in stage $(0,1)$, and for
layers $\ell < \log_2 n$, action $1$ always moves ``left'' in the
tree, and actions $2,\dots,A$ always moves ``right'' in the
tree. Moreover, the leaf-states are all absorbing. The construction is
given in Figure~\ref{fig:mdp_embed}.

The only part unknown to the learner are the transition vectors
$\{q_x\}_{x \in [n]}$, where $q_x(s,a)$ describes the probability of
transitioning to leaf $(s,1+\log_2n)$ when taking action $a$ from
state $(x,\log_2 n)$.
We now index rewards by $(x,\nu) \in [n]\times [0,1]^{2n}$, where
$r_{x,\nu}$ places action-independent reward $1$ on state
$(x,\log_2n)$, action-independent reward $\nu(s)$ on states
$(s,1+\log_2 n)$, and reward $0$ everywhere.


Assume that the transitions $q_x$ satisfy the near-uniformity
condition of \eqref{eq:q_condition} for $\epsilon = 1/4H$. Then, for
reward $r_{x,\nu}$, the high reward of $1$ at $(x,\log_2 n)$ forces
any near-optimal policy to visit $(x,\log_2 n)$ and subsequently play
near optimal actions at this state.  However, playing optimally at
$(x,\log_2 n)$ under reward $r_{x,\nu}$ for all $\nu$ is equivalent to
reward-free learning of a single instance of the construction from
Lemma~\ref{lem:informal_single_state_lb}. By varying $x \in [n]$ for
the reward vectors $r_{x,\nu}$, the learner is forced to learn $n$ such
instances, yielding the $\Omega(n \cdot n A/\epsilon^2) = \Omega(S^2
A/\epsilon^2)$ lower bound. This can be improved to $\Omega(H^2 S^2
A/\epsilon^2)$ by using the absorbing states to create a chain of
$\Omega(H)$ rewards.


%


\section{Conclusion}

In this paper, we propose a new ``reward-free RL'' framework,
comprising of two phases. In the exploration phase, the learner first
collects trajectories from an MDP $\mathcal{M}$ \emph{without}
receiving any reward information.  After the exploration phase, the
learner is no longer allowed to interact with the MDP and she is
instead tasked with computing near-optimal policies under for
$\mathcal{M}$ for a collection of given reward functions.  This
framework is particularly suitable when there are many reward
functions of interest, or when we are interested in learning the
transition operator directly.

This paper provides an efficient algorithm that conducts
$\widetilde{O}(S^2A\mathrm{poly}(H)/\epsilon^2)$ episodes of
exploration and returns $\epsilon$-suboptimal policies for an
\emph{arbitrary} number of adaptively chosen reward functions. Our
planning procedure can be instantiated by any black-box approximate
planner, such as value iteration or natural policy gradient. We also
give a nearly-matching $\Omega(S^2AH^2/\epsilon^2)$ lower bound,
demonstrating the near-optimality of our algorithm in this setting.

We close with some directions for future work. On the technical level,
an interesting direction is to understand the sample complexity for
reward-free RL with a pre-specified reward function that is unobserved
during the exploration phase. Our lower bound proofs requires the
agent to be able to optimize all possible reward functions, so it does
not directly apply to this potentially easier setting. Can we use
$\tilde{O}(SA\poly(H)/\epsilon^2)$ samples in the exploration phase to
achieve this goal?


Another interesting direction is to design reward-free RL algorithms
for settings with function approximation. We believe our work
highlights and introduces some mechanisms that may be useful in the
function approximation setting, such as the concept of significant
states (Definition~\ref{def:significant}) and the coverage
guarantee~\eqref{eq:d_cond}. How do we generalize these concepts to
the function approximation setting?

We hope to pursue these directions in future work. 


\bibliographystyle{plainnat}
\bibliography{ref}

\clearpage

\appendix

\section{The \zrmax algorithm}
\label{app:rmax}
\rmax is a well-known PAC exploration algorithm
\cite{brafman2002r}. Here, we show that a modified version of \rmax,
which we call \zrmax, addresses the reward-free exploration
setting. The difference between \zrmax and \rmax is that we set
the reward in ``known'' states to $0$ instead of the true reward,
which explains the name. We briefly describe the algorithm and derive
the PAC bound relying heavily on prior arguments. Details about \rmax
and its analysis can be found in prior
work~\cite{brafman2002r,kakade2003sample}.


Following the reward-free exploration framework proposed in
Section~\ref{prelim}, the \zrmax algorithm first collects samples
without knowledge about reward (exploration) and then computes a policy
for each configuration of reward function (planning). We define set of
known states $\mathcal{K}$ to be
\begin{align*}
\mathcal{K}:=\left\{ \left( s,h \right) :\forall a\in \mathcal{A},N_h\left( s,a \right) \ge m \right\}
\end{align*}
where $N_h\left( s,a \right)$ counts how many times $s$ has been
visited and $a$ was taken in the $h$-th step and $m$ is a parameter to
be specified later. The set $\mathcal{K}$ contains states that we have
visited enough times to estimate the corresponding transition kernel,
and is typically referred to as the ``known set'' in the
literature. For $(s,h)$ not in $\mathcal{K}$, we call them
``unknown.''

Now \zrmax explores as follows. In each episode $i \in [N]$, the agent has a known set $\mathcal{K}_i$ and
\begin{enumerate}
  \item builds an empirical MDP $\hat{\cM}_{i,\mathcal{K}_i}$ with parameters
    \begin{align}
  \mathbb{P}_h\left( \cdot |s,a \right) =\begin{cases}
	\hat{\mathbb{P}}_{h,i}\left( \cdot |s,a \right) \,\, \text{if }\left( s,h \right) \in \mathcal{K}_i\\
	\mathds{1}\left\{ s'=s \right\} \,\,\text{otherwise}\\
\end{cases} \qquad
r_h\left( s,a \right) =\begin{cases}
	\text{0 if }\left( s,h \right) \in \mathcal{K}_i\\
	\text{1 otherwise}\\
\end{cases}
\label{zero-one-reward}
\end{align}
where $\P_{h,i}$ is the empirical estimation of $\P_h$ in the $i$-th episode. 
\item computes $\pi_i = \pi_{\hat{\cM}_{i,\mathcal{K}_i}}^{\star}$ on $\hat{\cM}_{i,\mathcal{K}_i}$ by value iteration.
\item samples a trajectory from the environment following $\pi_i$. 
\item constructs $\mathcal{K}_{i+1}$ for the next episode
\end{enumerate}

For the planning phase, we first sample an index $i \in [N]$ uniformly
and construct the MDP $\hat{\cM}_{i,\mathcal{K}_i}$. Then given reward
function, we can just perform value iteration on
$\hat{\cM}_{i,\mathcal{K}_i}$, which gives us a near optimal policy.

\subsection{Analysis}
A central concept for analyzing the sample complexity of \zrmax is the
escape probability, which is the probability of visiting the unknown
states. Formally,
\begin{align*}
	p_{\mathcal{K}}^{\pi}=\mathbb{P}_{\cM, \pi}\left\{ \exists \left( s_h,h \right) \,\,s.t.\left( s_h,h \right) \notin \mathcal{K} \right\}  
\end{align*}
The above definition also depends on the corresponding MDP $\cM$. Since
we only care about the escape probability w.r.t the true MDP $\cM$, we
will omit this dependence. The key observation is that there cannot be
too many episodes where the escape probability is large. The inuition
is that, if the escape probability is big, then the agent will soon
visit an unknown states. However, the agent can visit unknown states
at most $mSA$ times in total.

\begin{lemma}[Lemma 8.5.2 in \cite{kakade2003sample}]
Let $\pi_i$ be the policy followed in the $i^{\textrm{th}}$episode and
$\mathcal{K}_i$ be corresponding set of known states. Then with
probability $1-p$, there can be at most $\cO\left(
\frac{mSA}{\varepsilon}\log \frac{SANH}{p} \right) $ episodes
where $p_{\mathcal{K}_i}^{\pi _i}>\varepsilon $.
\end{lemma}

As a result, we have the following corollary.
\begin{corollary}
If we sample $i$ uniformly from $1$ to $K$, then with probability
$1-p -\cO\left( \frac{mSA}{\varepsilon N}\log \frac{SANH}{p}
\right) $, we have $p_{\mathcal{K}_i}^{\pi _i}\le \varepsilon$.
\end{corollary}

In what follows, we focus on a single ``good'' episode $i$ where
$p_{\mathcal{K}_i}^{\pi_i} \le \varepsilon$. Since we focus on a
single episode, let us denote $\mathcal{K}_i$ by $\mathcal{K}$ and
$\pi_i$ by $\pi _{\hat{\cM}_{\mathcal{K}}}^{\star}$. There are three
MDPs of interest, with important details presented in
Table~\ref{MDP-compare}.

\begin{table}[t]
	\center
	\begin{tabular}{l|l|l|l}
			 & $\cM$  & $\cM_{\mathcal{K}}$        & $\hat{\cM}_{\mathcal{K}}$       \\ \hline
	Known ($\mathcal{K}$) & $=\cM$ & $=\cM$        & $\approx \cM$       \\ \hline
	Unknown  & $=\cM$ & self loop & self loop \\ 
	\end{tabular}
	\caption{A comparison between the three MDPs involved.}
	\label{MDP-compare}
\end{table}

$\cM$ is the true MDP of interest, that we will use to measure the
performance of the policy we find in the planning phase.
$\hat{\cM}_{\mathcal{K}}$ is the MDP we use for computing policies in
both exploration and planning phases. The final MDP, $\cM_{\mathcal{K}}$
is an intermediate MDP which agrees with $\cM$ on the known set but
follows self-loops in the unknown states.  Our plan is to prove with
high probability, the value of any policy$\pi$ on $\cM$ and
$\hat{\cM}_{\mathcal{K}}$ are close, which implies the desired sample
complexity result using the same argument as in
Theorem~\ref{thm:plan}. 

The first step is to prove that for any policy $\pi$, the values on
$\cM_{\mathcal{K}}$ and $\hat{\cM}_{\mathcal{K}}$ are similar.
\begin{lemma}
	\label{rmax-first}
With probability $1-p$, for any policy $\pi$ and reward function $r$,
\begin{align*}
\left| \E_{s_1 \sim \P_1}[V_{1,\hat{\cM}_{\mathcal{K}}}^{\pi}(s_1;r)-V_{1,\cM_{\mathcal{K}}}^{\pi}(s_1;r)] \right|\le \cO\left( H^2\sqrt{\frac{S}{m}\log \frac{SANH}{p}} \right).
\end{align*}
\end{lemma}
\begin{proof}
We apply Lemma~\ref{value_difference} to $\cM_{\mathcal{K}}$ and $\hat{\cM}_{\mathcal{K}}$, since the reward function is the same and the transition kernel is the same for unknown states,
\begin{align*}
	\left| \E_{s_1 \sim \P_1}[V_{1,\hat{\cM}_{\mathcal{K}}}^{\pi}(s_1;r)-V_{1,\cM_{\mathcal{K}}}^{\pi}(s_1;r)] \right|& \le \mathbb{E}_{M_{\mathcal{K}},\pi}\left\{ \sum_{h=1}^H{\mathds{1}\left\{ \left( s_h,h \right) \in \mathcal{K} \right\} |\left( \P_h -\hat{\P}_h \right)V_{h+1,\hat{\cM}_{\mathcal{K}}}^{\pi}}(s_h,a_h) |\right\} 
\\
& \le \cO\left( H^2\sqrt{\frac{S}{m}\log \frac{SANH}{p}} \right). \tag*\qedhere
\end{align*}
\end{proof}

The second step is to prove that for any policy $\pi$, the values on
$\cM_{\mathcal{K}}$ and $\cM$ are similar, which is less straightforward.
\begin{lemma}
With probability $1-p$ and $i$ is a ''good'' episode,
for any policy $\pi$,
\begin{align*}
	\left| \E_{s_1 \sim \P_1}[V_{1,\hat{\cM}_{\mathcal{K}}}^{\pi}(s_1;r)-V_{1,\cM_{\mathcal{K}}}^{\pi}(s_1;r)] \right|\le H^3\varepsilon +\cO\left( H^4\sqrt{\frac{S}{m}\log \frac{SANH}{p}} \right).
\end{align*}
\end{lemma}
\begin{proof}
Notice that for any policy $\pi$, if we can upper bound the escape
probability, then $\cM_{\mathcal{K}}$ and $\cM$ must be similar for this
policy. Fortunately, this is actually the case, due to our setting of
the reward function in the exploration phase,
following~\eqref{zero-one-reward}. Then by definition for any $s$,
\begin{equation*}
\E_{s_1 \sim \P_1}V_{\cM_{\mathcal{K}}}^{\pi}(s_1)\ge p_{\mathcal{K}}^{\pi}, \quad \text{and} \quad Hp_{\mathcal{K}}^\pi \ge \E_{s_1 \sim \P_1}V_{\cM_{\mathcal{K}}}^\pi(s_1).
\end{equation*}
and using Lemma~\ref{rmax-first}, 
\begin{align*}
	\E_{s_1 \sim \P_1}V_{\hat{\cM}_{\mathcal{K}}}^{\pi}(s_1)\ge p_{\mathcal{K}}^{\pi}-\cO\left( H^2\sqrt{\frac{S}{m}\log \frac{SANH}{p}} \right)   
\end{align*}
However, since we are considering a good episode, we know that for the
optimal policy on $\hat{\cM}_{\mathcal{K}}$, $\pi
_{\hat{\cM}_{\mathcal{K}}}^{*}$, we have $p_{\mathcal{K}}^{\pi _{\hat{\cM}_{\mathcal{K}}}^{*}}\le \varepsilon$. 
Therefore,
\begin{align*}
	&H\varepsilon +\cO\left( H^2\sqrt{\frac{S}{m}\log \frac{SANH}{p}} \right)
	\ge Hp_{\mathcal{K}}^{\pi _{\hat{\cM}_{\mathcal{K}}}^{*}}+\cO\left( H^2\sqrt{\frac{S}{m}\log \frac{SANH}{p}} \right) 
	\\
	\ge& \E_{s_1 \sim \P_1}V_{\cM_{\mathcal{K}}}^{\pi _{\hat{\cM}_{\mathcal{K}}}^{*}}(s_1)+\cO\left( H^2\sqrt{\frac{S}{m}\log \frac{SANH}{p}} \right) 
	\ge \E_{s_1 \sim \P_1}V_{\hat{\cM}_{\mathcal{K}}}^{\pi _{\hat{M}_{\mathcal{K}}}^{*}}(s_1)
	\ge \E_{s_1 \sim \P_1}V_{\hat{\cM}_{\mathcal{K}}}^{\pi}(s_1)
	\\
	\ge& p_{\mathcal{K}}^{\pi}-\cO\left( H^2\sqrt{\frac{S}{m}\log \frac{SANH}{p}} \right)  
\end{align*}
and as a result
\begin{align*}
p_{\mathcal{K}}^{\pi}\le H\varepsilon +\cO\left( H^2\sqrt{\frac{S}{m}\log \frac{SANH}{p}} \right). 
\end{align*}
Now notice $\cM_{\mathcal{K}}$ and $\cM$ are only different on unknown
states, which will not influence the agent unless the agent escapes
from $\mathcal{K}$. Using Lemma~\ref{value_difference} on
$\cM_{\mathcal{K}}$ and $\cM$ we have
\begin{align*}
	\left| \E_{s_1 \sim \P_1}[V_{1,\hat{\cM}_{\mathcal{K}}}^{\pi}(s_1;r)-V_{1,\cM_{\mathcal{K}}}^{\pi}(s_1;r)] \right|\le H^3\varepsilon +\cO\left( H^4\sqrt{\frac{S}{m}\log \frac{SANH}{p}} \right). \tag*\qedhere
\end{align*}
\end{proof}

Finally we can put everything together. Again following the argument
in Theorem~\ref{thm:plan}, we have
\begin{theorem}
With probability $1-2p -\cO\left( \frac{mSA}{\varepsilon K}\log
\frac{SANH}{p} \right) $, given any reward function, the \zrmax
algorithm can output a policy $\pi$ such that
\begin{align*}
	\E_{s_1 \sim \P_1}[V_{1,\cM}^{*}(s_1)-V_{1,\cM}^{\pi}(s_1)]\le H^3\varepsilon +\cO\left( H^4\sqrt{\frac{S}{m}\log \frac{SANH}{p}} \right). 
\end{align*}
\end{theorem}
Now we can set the parameters $m$ and $\varepsilon$. To make
$\E_{s_1 \sim \P_1}[V_{1,\cM}^{*}(s_1)-V_{1,\cM}^{\pi}(s_1)]\le \epsilon$, we need $m\ge
\Omega\left( \frac{SH^{8}}{\epsilon ^2}\log \frac{SAKH}{p}
\right) $ and $\varepsilon \le \cO\left( \epsilon/H^3
\right)$. This means we must set
\begin{align*}
N\ge \Omega \left( \frac{H^{11}S^2A}{\epsilon ^3p}\left( \log \frac{SANH}{p} \right) ^2 \right) 
\end{align*}
or equivalently,
\begin{align*}
N\ge \Omega \left( \frac{H^{11}S^2A}{\epsilon ^3p}\left( \log \frac{SAH}{p \epsilon} \right) ^2 \right)  
\end{align*}
This sample complexity is quite poor because it scales with
$\epsilon^{-3}$ and polynomially, rather than logarithmically, with
$1/p$.

\section{MaxEnt Exploration}
\label{app:max_ent}
Another approach for reward-free exploration was studied
in~\cite{hazan2018provably}. They consider the infinite horizon
discounted setting with discount factor $\gamma$, and they show that
with $\tilde{O}(\frac{S^2A}{\varepsilon^3(1-\gamma)^2})$ trajectories
of length $\tilde{O}(\frac{\log S}{\varepsilon^{-1}\log(1/\gamma)})$,
they can find a policy $\hat{\pi}$ such that
\begin{align*}
\frac{1}{S}\sum_{s} \log(d_{\hat{\pi}}(s)) \geq \max_{\pi} \frac{1}{S} \sum_s \log(d_{\pi}(s)) - \varepsilon
\end{align*}
where $d_\pi(s) = (1-\gamma)\sum_{t=1}^\infty \gamma^t d_{t,\pi}(s)$
and $d_{t,\pi}(s) = \Pr[s_t =s \mid \pi]$.

For reward free exploration, we want to use this guarantee to
establish a condition similar to the conclusion of
Theorem~\ref{thm:main_exp}. For the sake of contradiction, suppose
there exists some policy $\tilde{\pi}$ and some state $\tilde{s}$ such
that
\begin{align*}
\frac{d_{\tilde{\pi}}(\tilde{s})}{d_{\hat{\pi}}(\tilde{s})} > 4S.
\end{align*}
We want to show that the non-Markovian mixture policy
$(1-\alpha)\hat{\pi} + \alpha \tilde{\pi}$ for some $\alpha>0$
demonstrates that $\hat{\pi}$ violates its near-optimality guarantee
for the optimization problem. To do this, we lower bound the
difference in objective values between the mixture policy and
$\hat{\pi}$:
\begin{align*}
& \frac{1}{S}\sum_{s} \log((1-\alpha)d_{\hat{\pi}}(x) + \alpha
d_{\tilde{\pi}}(s)) - \log(d_{\hat{\pi}}(s)) = \frac{1}{S}\sum_{s}\log\left(1 - \alpha\frac{d_{\hat{\pi}(s)} - d_{\tilde{\pi}}(s)}{d_{\hat{\pi}}(s)}\right)\\
& \geq \frac{S-1}{S}\log(1-\alpha) + \frac{1}{S}\log(1 + \alpha(4S-1))\\
& \geq \frac{S-1}{S}\frac{-\alpha}{1-\alpha} + \frac{1}{S}\frac{\alpha(4S-1)}{1+\alpha(4S-1)}\\
& = \frac{\alpha}{S}\left(\frac{4S}{1+\alpha(4S-1)} - \frac{1}{1+\alpha(4S-1)} - \frac{(S-1)}{1-\alpha} \right).
\end{align*}
Here we are using that $\log(1-x_1+x_2)$ is monotonically increasing
in $x_2$ so we use the lower bound of $4S$ on $\tilde{s}$ and the
trivial lower bound of $0$ on all of the other states. We also use
that $\log(1+x) \geq \frac{x}{1+x}$, which holds for any $x > -1$.
The expression inside the parenthesis can be simplified to
\begin{align*}
\frac{3S + S\alpha - 4S^2\alpha}{(1-\alpha)(1+\alpha(4S-1))}.
\end{align*}
At this point we can see that if $\alpha \geq 1/S$ then this
expression is negative, so the mixture policy with large $\alpha$ does
not yield any improvement in objective. On the other hand, for any
$\alpha < 1/S$ then this inner expression is $\Theta(S)$. So if we set
$\alpha = \Theta(1/S)$ the overall improvement in objective is
$\Omega(1/S)$. This means that if we want establish the guarantee in
Theorem~\ref{thm:main_exp}, we must set $\varepsilon = 1/S$, at which
point the overall sample complexity scales with $S^5$, which is quite
poor.

Note that this calculation shows that $O(S^5)$ samples is sufficient
for the maximum entropy approach to find a suitable exploratory
policy, but we do not claim that it is necessary for this method. A
sharper analysis may be possible, but we are not aware of any such
results.


\section{Proof for Main Results}
In this section, we present proofs for results in Section \ref{sec:main_results}.

\subsection{Exploration Phase}
We begin with the proof of Lemma~\ref{zanette}, which is a simple modification of the Theorem 1 in \cite{zanette2019tighter}.
\begin{proof}[Proof of Lemma~\ref{zanette}]
WLOG, we can assume $s_1$ is fixed. This is because for $s_1$ stochastic from $\P_1$, we can simply add an artificial step before the first step of MDP, which always starts from the same state $s_0$, has only one action, and the transition to $s_1$ satisfies $\P_1$. This creates a new MDP with fixed initial state with length $H+1$, which is equivalent to the original MDP.

We use an alternative upper-bound for equation (156) in \cite{zanette2019tighter}, which gives:
\begin{align*}
  &\frac{1}{N_0H}\sum_{k=1}^{N_0}{\E_{\pi _k}\left[\left. ( \sum_{h=1}^H{r( s_h, a_h)}-V_{1}^{\pi _k}( s_1 ) ) ^2\right|s_1 \right]}
  \\
  \le& \frac{2}{N_0H}\sum_{k=1}^{N_0}{\E_{\pi _k}\left[\left. ( \sum_{h=1}^H{r( s_h, a_h)} ) ^2+( V_{1}^{\pi _k}( s_1 ) ) ^2\right|s_1 \right]}
  \\
  \stackrel{(i)}{\le}& \frac{2}{N_0H}\sum_{k=1}^{N_0}{\E_{\pi _k}\left[\left.  \sum_{h=1}^H{r( s_h, a_h )}  + V_{1}^{\pi _k}( s_1 ) \right|s_1 \right]}\\
  \le& \frac{4}{N_0H}\sum_{k=1}^{N_0} V_{1}^{\pi _k}( s_1 ) \le \frac{4}{H}V_{1}^{\star}( s_1 )  
\end{align*} 
where $\pi_k$ is the policy used in \EULER in the $k$-th episode. Step (i) is because using the reward function designed in Line~\ref{line:reward_def} in Algorithm~\ref{alg:main_exp}, we have all reward equal to zero except one state. Therefore, we have $\sum_{h=1}^H r( s_h, a_h )\le 1$ and $V_{1}^{\pi}( s_1 ) \le 1$. Therefore, we have replace the upper bound $\mathcal{G}^2$ in (156) of \cite{zanette2019tighter} by $4V_{1}^{\star}( s_1 ) $.

This allows us also replace the $\mathcal{G}^2$ in Theorem 1 of \cite{zanette2019tighter} by $4V_{1}^{\star}( s_1 )$, which gives the regret of algorithm (note \cite{zanette2019tighter} is for stationary MDP, while our paper is for non-stationary MDP, thus $S$ in \cite{zanette2019tighter} need to be replaced by $SH$ in our paper due to state augmentation, which creates new states as $(s, h)$):
\begin{equation*}
\sum_{k=1}^{N_0} [V^\star_1 (s_1) - V^{\pi_k}(s_1)] \le \tlO(\sqrt{V_{1}^{\star}( s_1 )SAT } + S^2AH^4)
\end{equation*}
Finally, plug in $T = N_0 H$, we finish the proof.
\end{proof} 

Now we can prove the main result in this section.
\begin{proof}[Proof of Theorem~\ref{thm:main_exp}]
In the following we can fix a state $(s,h)$ and consider the corresponding policy given by \EULER. Remember in our setting (Line~\ref{line:reward_def} in Algorithm~\ref{alg:main_exp}),  
$$
\E_{s_1\sim \P _1}V_{1}^{\star}( s_1 ) =\underset{\pi}{\max}P_{h}^{\pi}( s )
$$
Therefore the regret guarantee Lemma~\ref{zanette} implies
$$
\underset{\pi}{\max}P_{h}^{\pi}( s ) -\frac{1}{N_0}\sum_{\pi \in \Phi ^{( s,h )}}{P_{h}^{\pi}( s )}\le c_0 \sqrt{\frac{SAH\iota _0\cdot \max_{\pi} P_{h}^{\pi}( s )}{N_0}}+\frac{S^2AH^{4}\iota_0 ^3}{N_0}
$$
for some absolute constant $c_0$. Therefore, in order to make the following true
$$
\underset{\pi}{\max}P_{h}^{\pi}( s ) -\frac{1}{N_0}\sum_{\pi \in \Phi ^{( s,h )}}{P_{h}^{\pi}( s )}\le \frac{1}{2}\underset{\pi}{\max}P_{h}^{\pi}( s ) 
$$
We simply need to choose $N_0$ large enough so that:
\begin{align*}
  \sqrt{\frac{SAH\iota _0\cdot \max_{\pi} P_{h}^{\pi}( s )}{N_0}}&\le \underset{\pi}{c_1\cdot \max}P_{h}^{\pi}( s ) 
  \\
  \frac{S^2AH^{4}\iota_0 ^3}{N_0} &\le \underset{\pi}{c_1\cdot \max}P_{h}^{\pi}( s ) 
\end{align*}
for a sufficient small absolute constant $c_1$. Combining with the fact that for $\text{~$\delta$-significant~} (s,h)$, $\max_{\pi} P_{h}^{\pi}( s ) \ge \delta$, we know choosing $N_0=\cO(S^2AH^4\iota_0 ^3/\delta)$ is sufficient. As a result, we have
$$
\max_{\pi} \frac{P_{h}^{\pi}( s )}{ \frac{1}{N_0}\sum_{\pi \in \Phi ^{( s,h )}} P_{h}^{\pi}( s )}\le 2
$$
Since Algorithm \ref{alg:main_exp} sets all policy in $\Phi ^{( s,h )}$ to choose action uniformly randomly at $(s, h)$, this implies
$$
\max_{\pi,a} \frac{P_{h}^{\pi}( s,a )}{\frac{1}{N_0}\sum_{\pi \in \Phi ^{( s,h )}}P_{h}^{\pi}( s,a )}\le 2A
$$
Finally, we can apply the same argument for all $\delta$-significant $(s, h)$, and let $\Psi = \cup\{\Phi^{(s, h)}\}_{(s, h)}$ which gives:
\begin{equation*}
\forall \text{~$\delta$-significant~} (s,h), \quad \max_{\pi, a}\frac{P_{h}^{\pi}(s,a)}{\frac{1}{N_0 SH}\sum_{\pi \in \Psi}{P_{h}^{\pi}( s,a )}} \le 2SAH.
\end{equation*}
This finishes the proof.
\end{proof}

\subsection{Planning Phase}

The following lemma (E.15 in \cite{dann2017unifying}) will be useful to characterize the difference between $V_{h}^{\pi}( s;r)$ and $\hat{V}_{h}^{\pi}(s;r) $ .

\begin{lemma}[Lemma E.15 in \cite{dann2017unifying}]
  \label{value_difference}
  For any two MDPs $\cM'$ and $\cM''$ with rewards $r'$ and $r''$
and transition probabilities $\P'$ and $\P''$, the difference in values $V'$, $V''$ with respect to the same policy $\pi$ can
be written as
\begin{align*}
V'_{h}( s ) -V''_{h}( s ) =\E_{\cM'',\pi}&\left[\left. \sum_{i=h}^H[ r'_i( s_i,a_i ) -r''_i( s_i,a_i ) + (\P'_i -\P''_i) V'_{i+1}( s_i,a_i )]\right|s_h=s \right]
\end{align*}
\end{lemma}

With this decomposition in mind, we can prove Lemma~\ref{lem:plan}.

\begin{proof}[Proof of Lemma~\ref{lem:plan}]

In this section, we always use $\E$ to denote the expectation under the true MDP $\cM$.   Using Lemma~\ref{value_difference} on $\cM$ (the true MDP) and $\hat{\cM}$ (the empirical version), we have
\begin{equation*}
| \E_{s_1\sim \P _1}\{ \hat{V}_{1}^{\pi}(s_1;r)-V_{1}^{\pi}(s_1;r) \} | \le | \E_\pi\sum_{h=1}^H(\hat{\P}_h -\P_h)\hat{V}_{h+1}^{\pi}( s_h,a_h ) |
\le \E_\pi \sum_{h=1}^H |(\hat{\P}_h -\P_h)\hat{V}_{h+1}^{\pi}( s_h,a_h ) |
\end{equation*}


Let $\mathcal{S}_{h}^{\delta}:=\{ s:\underset{\pi}{\max}P_{h}^{\pi}( s ) \ge \delta \}$ be the set of $\delta$-significant states in the $h$-th step. We further have:
\begin{equation*}
\E_\pi |(\hat{\P}_h -\P_h)\hat{V}_{h+1}^{\pi}( s_h,a_h ) |
\le \underbrace{\sum_{a,s\in \mathcal{S}_{h}^{\delta}}| (\hat{\P}_h -\P_h)\hat{V}_{h+1}^{\pi}( s,a ) |P^{\pi} _h( s,a )}_{\xi_h}+
\underbrace{\sum_{a,s\notin \mathcal{S}_{h}^{\delta}}{| (\hat{\P}_h -\P_h)\hat{V}_{h+1}^{\pi}( s,a ) |P^{\pi} _h( s,a )}}_{\zeta_h}
\end{equation*}
By definition of insignificant state, we have:
\begin{equation}\label{eq:pf_insignificant}
\zeta_h \le H\sum_{a,s\notin \mathcal{S}_{h}^{\delta}}{P^{\pi} _h( s,a )} = H\sum_{s\notin \mathcal{S}_{h}^{\delta}}{P^{\pi} _h( s )} \le H\sum_{s\notin \mathcal{S}_{h}^{\delta}}{\delta} \le HS\delta.
\end{equation}
On the other hand, by Cauchy-Shwartz inequality, we have:
\begin{equation*}
\xi_h \le \left[\sum_{a,s\in \mathcal{S}_{h}^{\delta}}{| (\hat{\P}_h -\P_h)\hat{V}_{h+1}^{\pi}( s,a ) |^2P^{\pi} _h( s,a )}\right]^{\frac{1}{2}}
=\left[\sum_{a,s\in \mathcal{S}_{h}^{\delta}}{| (\hat{\P}_h -\P_h)\hat{V}_{h+1}^{\pi}( s,a ) |^2P^{\pi}_h(s)\pi_h(a|s)}\right]^{\frac{1}{2}}
\end{equation*}
We note since $\hat{V}_{h+1}^{\pi}$ only depends on $\pi$ at $h+1, \cdots, H$ steps, it does not depends on $\pi_h$. Therefore, we have:
\begin{align*}
\sum_{a,s\in \mathcal{S}_{h}^{\delta}}{| (\hat{\P}_h -\P_h)\hat{V}_{h+1}^{\pi}( s,a ) |^2P^{\pi}_h(s)\pi_h(a|s)}
\le& \max_{\pi'_h}\sum_{a,s\in \mathcal{S}_{h}^{\delta}}{| (\hat{\P}_h -\P_h)\hat{V}_{h+1}^{\pi}( s,a ) |^2P^{\pi}_h(s)\pi'_h(a|s)}\\
=& \max_{\nu:\cS \rightarrow \cA} \sum_{a,s\in \mathcal{S}_{h}^{\delta}} | (\hat{\P}_h -\P_h)\hat{V}_{h+1}^{\pi}( s,a ) |^2P^{\pi}_h(s)\mathds{1}\{ a=\nu( s ) \}
\end{align*}
where the last step is because the maximization over $\pi'_h$ achieves at deterministic polices.

Recall that by preconditions, we have \ref{eq:d_cond} holds for $\delta = \epsilon/(2SH^2)$. That is, for any $s\in \mathcal{S}_{h}^{\delta}$ we always have
$$
\max_{\tilde{\pi}}\frac{P^{\tilde{\pi}}_h( s,a )}{\mu _h( s,a )}\le 2SAH
$$
Therefore, for any $(s, a)$ pair, we can design a policy $\pi'$ so that $\pi'_{h'} = \pi_{h'}$ for all $h' < h$, and $\pi'_{h}(s)= a$. This will give that
\begin{equation*}
P^\pi_h(s) = P^{\pi'}_h(s) = P^{\pi'}_h(s, a) \le 2SAH\mu_h(s, a)
\end{equation*}
which gives:
\begin{align*}
&\sum_{a,s\in \mathcal{S}_{h}^{\delta}} | (\hat{\P}_h -\P_h)\hat{V}_{h+1}^{\pi}( s,a ) |^2P^{\pi}_h(s)\mathds{1}\{ a=\nu( s ) \} \\
\le& 2SAH \sum_{a,s\in \mathcal{S}_{h}^{\delta}} | (\hat{\P}_h -\P_h)\hat{V}_{h+1}^{\pi}( s,a ) |^2\mu_h(s)\mathds{1}\{ a=\nu( s ) \}\\
\le& 2SAH \sum_{s, a} | (\hat{\P}_h -\P_h)\hat{V}_{h+1}^{\pi}( s,a ) |^2\mu_h(s)\mathds{1}\{ a=\nu( s ) \}\\
=& 2SAH \E_{\mu_h}| (\hat{\P}_h -\P_h)\hat{V}_{h+1}^{\pi}( s,a ) |^2\mathds{1}\{ a=\nu( s ) \}
\end{align*}

By Lemma \ref{lem:sharp_concentration}, we have:
\begin{equation*}
\E_{\mu_h}| (\hat{\P}_h -\P_h)\hat{V}_{h+1}^{\pi}( s,a ) |^2\mathds{1}\{ a=\nu( s ) \} \le \cO\left( \frac{H^2S}{N}\log ( \frac{AHN}{p} ) \right) 
\end{equation*}
Therefore, combine all equations above, we have
$$
| \E_{s_1\sim \P _1}\{ \hat{V}_{1}^{\pi}( s_1;r ) -V_{1}^{\pi}( s_1;r ) \} |\le \cO( \sqrt{\frac{H^5S^2A}{N}\log ( \frac{AHN}{p} )} ) +H^2S\delta 
$$
Recall our choice $\delta =\epsilon /(2SH^2)$ and $N\ge c\frac{H^5S^2A}{\epsilon ^2}\log ( \frac{SAH}{p\epsilon} )$ 
for sufficiently large absolute constant $c$, which finishes the proof.
\end{proof}

\begin{lemma}\label{lem:sharp_concentration}
Suppose $\hat{\P}$ is the empirical transition matrix formed by sampling according to $\mu$ distribution for $N$ samples, then with probability at least $1-p$, we have for any $h\in[H]$:
\begin{equation*}
 \max_{G:\cS \rightarrow [0, H]} \max_{\nu:\cS \rightarrow \cA}
\E_{\mu_h}| (\hat{\P}_h -\P_h)G( s,a ) |^2\mathds{1}\{ a=\nu( s ) \} \le \cO\left( \frac{H^2S}{N}\log ( \frac{AHN}{p} ) \right) 
\end{equation*}
\end{lemma}

\begin{proof}
Define random variable
$$
X_{i}=(\hat{\P}_hG(s_i,a_i)  -G( s'_{i} ) ) ^2-(\P_hG(s_i,a_i)  -G( s'_{i} ) ) ^2
$$
where $( s_i,a_i,s'_{i} ) \sim \mu_h \times \P _h( \cdot |s_i,a_i ) $ is the $i$-th sample in level $h$ we collect.

Also we define
$$
Y_{i}=X_{i}\mathds{1}\{ a_i=\nu ( s_i ) \}.
$$
To simplify the notation, when some property of $Y_{i}$ holds for any $i$, we just use the notation $Y$ to describe a generic $Y_{i}$.

We first state some properties of the random variables $Y_{i}$, which are justified at the end of the proof.

\begin{itemize}
  \item (Expection) $
  \E Y=\E_{\mu_h}| (\hat{\P}_h -\P_h)G( s,a ) |^2\mathds{1}\{ a=\nu( s ) \}$
\item (Empirical risk minimization) $\sum_{i=1}^N{Y_{i}}\le 0$
\item (Self-bounded) $\text{Var}\{ Y \} \le 4H^2 \E Y$
\end{itemize}

Given these three properties, now we are ready to apply Berstein's inequality to $(\sum_{i=1}^N{Y_{i}})/N$. Since we are taking maximum over $\nu$ and $G(s)$ and $\hat{\P}$ is random, we need to cover all the possible values of $\hat{\P}G(s,a) \mathds{1}\{ a=\nu ( s ) \}$ and $\P G(s,a) \mathds{1}\{ a=\nu ( s ) \}$ to $\varepsilon $ accuracy to make Bernstein's inequality hold. For $\nu$, there are $A^S$ deterministic policies in total. Given a fixed $\nu$, $\hat{\P}G(s,a) \mathds{1}\{ a=\nu ( s ) \}$ and $\P G(s,a) \mathds{1}\{ a=\nu ( s ) \}$ can be covered by $( H/\varepsilon ) ^{2S}$ values by boundedness condition because for $a \ne  \nu(s)$ they are always 0. The overall approximation error will be at most $12H\varepsilon$ by boundedness condition.

As a result, with probability at least $1-p/H$, for any $\nu$, $G(s)$ and $\hat{\P}$, 

\begin{align*}
  &\E_{\mu_h}| (\hat{\P}_h -\P_h)G( s,a ) |^2\mathds{1}\{ a=\nu( s ) \}
=\E Y
\le \E Y-\frac{1}{N}\sum_{i=1}^N{Y_{i}}
\\
\le& \sqrt{\frac{\text{2Var}\{ Y \} \log ( ( \frac{H}{\varepsilon} ) ^{2S}\cdot A^S \cdot\frac{H}{p} )}{N}}+\frac{H^2\log ( ( \frac{H}{\varepsilon} ) ^{2S}\cdot A ^S \cdot \frac{H}{p} )}{3N}+12H\varepsilon
\\
\le &\sqrt{\frac{\text{2Var}\{ Y \} [ 2S\log ( \frac{HA}{\varepsilon} ) +\log \frac{H}{p} ]}{N}}+\frac{H^2[ 2S\log ( \frac{HA}{\varepsilon} ) +\log \frac{H}{p} ]}{3N}+12H\varepsilon
\end{align*}

We can simply choose $\varepsilon =HS/36N$ and thus
\begin{align*}
  &\E_{\mu_h}| (\hat{\P}_h -\P_h)G( s,a ) |^2\mathds{1}\{ a=\nu( s ) \} 
  \\
  \le& \sqrt{8H^2\E_{\mu_h}| (\hat{\P}_h -\P_h)G( s,a ) |^2\mathds{1}\{ a=\nu( s ) \} \frac{2S\log ( \frac{36AN}{S} ) +\log \frac{H}{p}}{N}}+\frac{H^2[ 2S\log ( \frac{36AN}{S} ) +\log \frac{H}{p}+S ]}{3N}
\end{align*}
Solving this quadratic formula we get
$$
\E_{\mu_h}| (\hat{\P}_h -\P_h)G( s,a ) |^2\mathds{1}\{ a=\nu( s ) \}\le \cO( \frac{H^2S}{N}\log ( \frac{ANH}{p} ) ) 
$$

Since the above upper bound holds for arbitrary $\nu$, $G(s)$ and $\P_h$, 
\begin{equation*}
  \max_{G:\cS \rightarrow [0, H]} \max_{\nu:\cS \rightarrow \cA}
 \E_{\mu_h}| (\hat{\P}_h -\P_h)G( s,a ) |^2\mathds{1}\{ a=\nu( s ) \} \le \cO\left( \frac{H^2S}{N}\log ( \frac{AHN}{p} ) \right) 
 \end{equation*}
Taking union bound w.r.t. $h$, the claim holds for any $h$ with probability $1-p$.

~

Finally we give the proofs for the claimed three properties of $Y_i$. We begin with the expectation property:
\begin{align*}
\E Y=&\E_{s,a \sim \mu _h}\E_{s' \sim\P _h( \cdot |s,a )}\{ \mathds{1}\{ a=\nu ( s ) \} [(\hat{\P}_hG(s,a)  -G( s' ) ) ^2-(\P_hG(s,a)  -G( s' ) ) ^2] \}
\\
\overset{( i )}{=}&2\E_{s,a \sim \mu _h}\E_{s' \sim\P _h( \cdot |s,a )}\{ \mathds{1}\{ a=\nu( s ) \} (\hat{\P}_h -\P_h)G( s,a ) (\P_hG(s,a)  -G( s' ) ) \}
\\
&+\E_{\mu_h}| (\hat{\P}_h -\P_h)G( s,a ) |^2\mathds{1}\{ a=\nu( s ) \}
\\
\overset{( ii )}{=}&\E_{\mu_h}| (\hat{\P}_h -\P_h)G( s,a ) |^2\mathds{1}\{ a=\nu( s ) \}
\end{align*}
where $(i)$ is by $b^2-d^2=( b-d+d ) ^2-d^2=( b-d ) ^2+2b( d-b )$ with $b=\hat{\P}_hG(s,a)  -G( s' ) $ and $d=\P_hG(s,a)  -G( s' )$ and  $(ii)$ is because $\E_{s' \sim\P _h( \cdot |s,a )}\{ G( s' )\}=\P_hG(s,a)$.

The emipirical risk minimization property is true because the evaluation rule is essentially minimizing the empirical Bellman error for each $(s,a)$ pair separately. Mathematically,
$$
\hat{\P}_hG(s,a) =\underset{g}{\text{arg}\max}\sum_{i=1}^N{\mathds{1}\{ s_i=s,a_i=a \} ( g-G( s' ) ) ^2}
$$

The self-bounded property is because
\begin{align*}
  &\text{Var}\{ Y \} \le \E( Y ) ^2
\\
\overset{( i )}{=}&\E\{ \mathds{1}\{ a=\nu ( s ) \} [ (\hat{\P}_h -\P_h)G( s,a ) ] ^2[ (\hat{\P}_h +\P_h)G( s,a ) -2G( s' ) ] ^2 \}
\\
\le& 4H^2\E_{\mu_h}| (\hat{\P}_h -\P_h)G( s,a ) |^2\mathds{1}\{ a=\nu( s ) \}
\\
=&4H^2 \E Y
\end{align*}
where $(i)$ by $b^2-d^2=( b+d ) ( b-d ) $ with $b=\hat{\P}_hG(s,a)  -G( s' ) $ and $d=\P_hG(s,a)  -G( s' )$.
\end{proof}

\subsection{Proof of Theorem~\ref{thm:main}}

Putting everything together we can prove the main theorem.
\begin{proof}[Proof of Theorem~\ref{thm:main}]
  We only need to choose the parameter $\delta$ and $N_0$. From the proof of Lemma~\ref{lem:plan} we can see, we need $\delta =\epsilon/(2SH^2)$ and thus $N_0\ge cS^{3}AH^6 \iota ^3/\epsilon$. Since we need $N_0$ episodes for each $(s,h)$, the total number episodes required for finding $\Psi$ is $\cO( cS^{4}AH^7 \iota ^3/\epsilon) $, which gives the second term in \eqref{equ:main}. The proof is completed by combining Theorem~\ref{thm:plan}, which gives the first term in \eqref{equ:main}.
  \end{proof}

\subsection{Approximate MDP Solvers}
The convergence of NPG is well studied in \cite{agarwal2019optimality} (tabluar \& infinite horizon) and \cite{cai2019provably} (linear approximation). However, the episodic setting has some unique characters (For example, we not every state can be arrive at the first step and the corresponding analysis in \cite{agarwal2019optimality} does not apply). Therefore the guarantee given in Proposition~\ref{prop:NPG} is different.

Since we only need to prove the guarantee on the true MDP, we will not distinguish true MDP $\cM$ and estimated MDP $\hat{\cM}$ here. Remember the NPG is defined by
  $$
\pi_h ^{( 0 )}( a|s ) =\text{1/}A
$$
and
$$
\pi _{h}^{( t+1 )}( a|s ) =\pi _{h}^{( t )}( a|s ) \exp \{ \eta ( Q_{h}^{( t )}( s,a ) -V_{h}^{( t )}( s ) ) \} /Z_{h}^{( t )}( s ) 
$$
where $Q_{h}^{( t )}( s,a ) :=Q_{h}^{\pi ^{( t )}}( s,a ) $ is computed following the value iteration procedure. Similarly we define $V_{h}^{( t )}( s ) :=V_{h}^{\pi ^{( t )}}( s ) $. The normalization constant can be written explicitly as
  $$
  Z_{h}^{( t )}( s ) :=\sum_{a\in \mathcal{A}}{\pi _{h}^{( t )}( a|s ) \exp \{ \eta [ Q_{h}^{( t )}( s,a ) -V_{h}^{( t )}( s ) ] \}}
$$
Notice the definition of the normalization constant is not unique. Here we choose the form that makes the following proof simpler but different choice will essentially gives exactly the same algorithm.

We begin with a lemma showing that the value function monotonically increases.
\begin{lemma}[Lemma 5.8 in \cite{agarwal2019optimality}]
  \label{monotone}
  Following the NPG iterations,
  $$
  \E_{s_1\sim \P _1}\{ V_{1}^{( t+1 )}( s_1;r ) -V_{1}^{( t )}( s_1;r ) \} \ge \frac{1}{\eta}\sum_{h=1}^H{\E_{s_h\sim M,\pi ^{( t+1 )}}\{ \log Z_{h}^{( t )}( s_h ) \}}\ge 0
$$
\end{lemma}
\begin{proof}
  By performance difference lemma \cite{kakade2002approximately},
  \begin{align*}
    &\E_{s_1\sim \P _1}\{ V_{1}^{( t+1 )}( s_1;r ) -V_{1}^{( t )}( s_1;r ) \}
    \\
    =&\sum_{h=1}^H{\E_{\pi ^{( t+1 )}}\{ \sum_{a\in \mathcal{A}}{\pi _{h}^{( t+1 )}( a|s ) [ Q_{h}^{( t )}( s,a ) -V_{h}^{( t )}( s ) ]} \}}
\\
=&\frac{1}{\eta}\sum_{h=1}^H{\E_{\pi ^{( t+1 )}}\{ \sum_{a\in \mathcal{A}}{\pi _{h}^{( t+1 )}( a|s_h ) \log \frac{\pi _{h}^{( t+1 )}( a|s_h ) Z_{h}^{( t )}( s_h )}{\pi _{h}^{( t )}( a|s_h )}} \}}
\\
=&\frac{1}{\eta}\sum_{h=1}^H{\E_{\pi ^{( t+1 )}}\{ \text{KL}( \pi _{h}^{( t+1 )}( s_h ) ||\pi _{h}^{( t )}( s_h ) ) +\log Z_{h}^{( t )}( s_h ) \}}
\\
\ge& \frac{1}{\eta}\sum_{h=1}^H{\E_{\pi ^{( t+1 )}}\{ \log Z_{h}^{( t )}( s_h ) \}}
\\
\overset{( i )}{\ge}&0
  \end{align*}
  where $(i)$ is by
  \begin{align*}
    \log Z_{h}^{( t )}( s ) =&\log \{ \sum_{a\in \mathcal{A}}{\pi _{h}^{( t )}( a|s ) \exp \{ \eta [ Q_{h}^{( t )}( s,a ) -V_{h}^{( t )}( s ) ] \}} \}
  \\
  \ge& \eta \sum_{a\in \mathcal{A}}{\pi _{h}^{( t )}( a|s ) [ Q_{h}^{( t )}( s,a ) -V_{h}^{( t )}( s ) ]}
  \\
  =&0
  \end{align*}
  because $V_{h}^{( t )}( s ) =\sum_{a\in \mathcal{A}}{\pi _{h}^{( t )}( a|s ) Q_{h}^{( t )}( s,a )}
$ by definition.
\end{proof}

Equipped with the monotone property, we can simply prove an upper bound for the cumulative regret, which immediately implies the convergence rate for the last iteration.

\begin{proof}[Proof of Proposition~\ref{prop:NPG}]
  Again by performance difference lemma,
\begin{align*}
  &\E_{s_1\sim \P _1}\{ V_{1}^{{\star}}( s_1;r ) -V_{1}^{( t )}( s_1;r ) \} 
  \\
  =  &\sum_{h=1}^H{\E_{\pi ^{\star}}\{ \sum_{a\in \mathcal{A}}{\pi _{h}^{{\star}}( a|s ) [ Q_{h}^{( t )}( s,a ) -V_{h}^{( t )}( s ) ]} \}}
  \\
  =&\frac{1}{\eta}\sum_{h=1}^H{\E_{\pi ^{\star}}\{ \sum_{a\in \mathcal{A}}{\pi _{h}^{{\star}}( a|s_h ) \log \frac{\pi _{h}^{( t+1 )}( a|s_h ) Z_{h}^{( t )}( s_h )}{\pi _{h}^{( t )}( a|s_h )}} \}}
  \\
  =&\frac{1}{\eta}\sum_{h=1}^H{\E_{\pi ^{\star}}\{ \text{KL}( \pi _{h}^{{\star}}( s_h ) ||\pi _{h}^{( t )}( s_h ) ) -\text{KL}( \pi _{h}^{{\star}}( s_h ) ||\pi _{h}^{( t+1 )}( s_h ) ) +\log Z_{h}^{( t )}( s_h ) \}}
\end{align*}

Now we can upper bound the regret of $\pi^{(T-1)}$ by upper bound the cumulative regret using Lemma~\ref{monotone}
\begin{align*}
  &\E_{s_1\sim \P _1}\{ V_{1}^{{\star}}( s_1;r ) -V_{1}^{( T-1 )}( s_1;r ) \}
  \\
  \le& \frac{1}{T}\sum_{t=0}^{T-1}{\E_{s_1\sim \P _1}\{ V_{1}^{{\star}}( s_1;r ) -V_{1}^{( t )}( s_1;r ) \}}
  \\
  \le& \frac{1}{\eta T}\sum_{t=0}^{T-1}{\sum_{h=1}^H{\E_{\pi ^{\star}}\{ \text{KL}( \pi _{h}^{{\star}}( s_h ) ||\pi _{h}^{( t )}( s_h ) ) -\text{KL}( \pi _{h}^{{\star}}( s_h ) ||\pi _{h}^{( t+1 )}( s_h ) ) +\log Z_{h}^{( t )}( s_h ) \}}}
  \\
  \le& \frac{1}{\eta T}\sum_{h=1}^H{\E_{\pi ^{\star}}\{ \text{KL}( \pi _{h}^{{\star}}( s_h ) ||\pi _{h}^{( 0 )}( s_h ) ) \}}+\frac{1}{\eta T}\sum_{t=0}^{T-1}{\E_{\pi ^{\star}}\{ \log Z_{h}^{( t )}( s_h ) \}}
  \\
  \le& \frac{H\log A}{\eta T}+\frac{1}{\eta T}\sum_{t=0}^{T-1}{\E_{\pi ^{\star}}\{ \log Z_{h}^{( t )}( s_h ) \}}
\end{align*}

Therefore we only need to bound $\log Z_{h}^{( t )}( s_h )$, where the technique in \cite{agarwal2019optimality} does not apply and we use a different approach. Notice for $x\le 1$,  $\exp \{ x \} \le 1+x+x^2$. So as long as $\eta \le \frac{1}{H}$,  $\eta [ Q_{h}^{( t )}( s,a ) -V_{h}^{( t )}( s ) ] \le 1$ and we have
\begin{align*}
  \log Z_{h}^{( t )}( s ) =&\log \{ \sum_{a\in \mathcal{A}}{\pi _{h}^{( t )}( a|s ) \exp \{ \eta [ Q_{h}^{( t )}( s,a ) -V_{h}^{( t )}( s ) ] \}} \}
  \\
  \le& \log \{ \sum_{a\in \mathcal{A}}{\pi _{h}^{( t )}( a|s ) \{ 1+\eta [ Q_{h}^{( t )}( s,a ) -V_{h}^{( t )}( s ) ] +\eta ^2[ Q_{h}^{( t )}( s,a ) -V_{h}^{( t )}( s ) ] ^2 \}} \}
  \\
  =&\log \{ 1+\eta ^2\sum_{a\in \mathcal{A}}{\pi _{h}^{( t )}( a|s ) [ Q_{h}^{( t )}( s,a ) -V_{h}^{( t )}( s ) ] ^2} \}
  \\
  \le& \eta ^2\sum_{a\in \mathcal{A}}{\pi _{h}^{( t )}( a|s ) [ Q_{h}^{( t )}( s,a ) -V_{h}^{( t )}( s ) ] ^2}
  \\
  \le& \eta ^2H^2
\end{align*}
Put everything together we have
$$
\E_{s_1\sim \P _1}\{ V_{1}^{{\star}}( s_1;r ) -V_{1}^{( T-1 )}( s_1;r ) \}\le \frac{H\log A}{\eta T}+\eta H^2
$$
This finishes the proof.
\end{proof}



\newcommand{\Xsingle}{\mathcal{X}_{\mathrm{single}}}
\newcommand{\Envsingle}{\Envir_{\mathrm{single}}}
\newcommand{\Psingle}{\Pclass_{\mathrm{single}}}
\newcommand{\nubar}{\overline{\nu}}
\newcommand{\Rsingle}{\Rclass_{\mathrm{single}}}
\newcommand{\Msingle}{\Mclass_{\mathrm{single}}}
\newcommand{\Envembed}{\Envir_{\mathrm{embed}}}
\newcommand{\Prsingle}{\Pr_{\mathrm{single}}}

\section{Proof of Lower Bound\label{app:proof_of_lower_bound}}

In this section, we prove our lower bound, Theorem~\ref{thm:main_lb}. First, we develop further notation in Section~\ref{ssec:lb_preliminaries} which will aid in distinguishing between multiple possible instances. Next, Section~\ref{ssec:lb_single_instance} states Lemma~\ref{lem:single_instance_formal}, the formal analogue of Lemma~\ref{lem:informal_single_state_lb}, which describes a lower bound for learning transitions at a single state. Then, Section~\ref{ssec:lb_nstates} embeds the construction to obtain an instance where the learner to learn transitions at $n$ states,  yielding the lower bound Theorem~\ref{thm:main_lb}. Finally, Section~\ref{ssec:prf_one_state_lb} details the proof of the $1$-state lower bound, Lemma~\ref{lem:informal_single_state_lb}.

\subsection{Preliminaries\label{ssec:lb_preliminaries}}
\paragraph{Environments, Transition Classes, Reward Classes}

	To formalize our embedding a one-state instance into a larger MDP, the following formalities are helpful: we define an environment $\Envir = (\calX,A,H)$ as a triple specifying a finite state space $\calX$, number of actions $A$, and horizon $H$. For a fixed environment, a transition class $\Pclass$ is a class of transition and initital state distributions, denoted by $\Pr$; a reward class $\Rclass$ is a family of reward functions $r: (\calX,A) \to [0,1]$. Given a reward vector $r$ and transition vector $\Pr$, we let $\mdp(\Pr,r)$ denote the with-reward MDP induced by $\Pr$ and $r$. We denote value of a policy $\pi$ on $\mdp(\Pr,r)$ by $V^{\pi}(\Pr,r)$.

	\paragraph{Reward-Free MDP Algorithm} A reward-free MDP algorithm $\Alg$ is algorithm which collects a random number $K$ trajectories from a given reward-free MDP, and then, when given a sequence of reward vectors $r^{(1)},r^{(2)},\dots,r^{(N)}$, returns a sequence of policies $\pi^{(1)},\pi^{(2)},\dots, \pi^{(N)}$.  We let $\Exp_{\Pr,\Alg}[\cdot]$ denote the expectation under the joint law prescribed by the explortion phase of algorithm $\Alg$ and transition operator $\Pr$.

	\paragraph{Correctness} Given $\epsilon,p \in (0,1)$, say that a reward-free MDP algorithm $(\epsilon,p,)$-learns a a problem class $\Mclass := (\Envir,\Rclass,\Pclass)$ if, for any transition operator $\Pr \in \Pclass$,  for any finite sequence of reward vectors $r^{(1)},\dots,r^{(N)} \in \Rclass$, $\Alg$ returns a sequence policies $\pi^{(1)},\dots, \pi^{(N)}$, such that, with probability $1-p$, the following holds
	\begin{align*}
	V^{\pi^{(i)}}(\Pr,r^{(i)}) \ge \max_{\pi} V^{\pi}(\Pr,r^{(i)}) - \epsilon, \quad \forall i \in [N].
	\end{align*}
	For the lower bound, we allow the policies $\pi$ prescribed by $\Alg$ to be arbitrary randomized mappings form observed histories, that is, $\Alg$ selects a random seed $\xi$ from some distribution; that is the policy at stage $h$ is a map
	\begin{align*}
	\pi_h: (s_1,\dots,s_h,a_1,\dots,a_{h-1},\xi) \to [A].
	\end{align*}

\subsection{Learning A Single Instance \label{ssec:lb_single_instance}}
	In this section, we define a triple $(\Envir,\Rclass,\Pclass)$ on $\BigOh{n}$-states which forces the learner to spend $\Omega(nA/\epsilon^2)$ trajectories to learn the transition probabilities at a given state.

	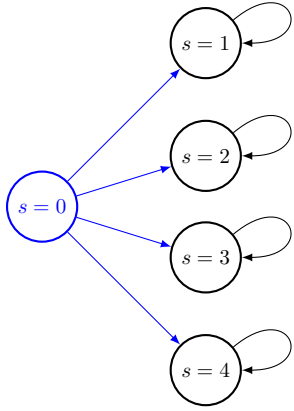
\begin{figure}
        \centering

 \begin{tikzpicture}[auto,node distance=8mm,>=latex, font = \small,scale=0.8, every node/.style={scale=0.8}]

    \tikzstyle{round}=[thick,draw=black,circle]

    \node[round,color = blue] (s0) {$s=0$};
    \node[round,above right=15mm and 15mm of s0,font = \small] (s1) {$s= 1$};
    \node[round,above right=0mm and 15mm of s0,font = \small] (s2) {$s=2$};
    \node[round,below right=0mm and 15mm of s0,font = \small] (s3) {$s= 3$};
    \node[round,below right=15mm and 15mm of s0,font = \small] (s4) {$s=4$};

    \draw[->,color = blue] (s0) -- (s1);
    \draw[->,color = blue] (s0) -- (s2);
    \draw[->,color = blue] (s0) -- (s3);
    \draw[->,color = blue] (s0) -- (s4);
    \draw[->] (s1) [out=40,in=0,loop] to coordinate[pos=0.1](aa) (s1);
    \draw[->] (s2) [out=40,in=0,loop] to coordinate[pos=0.1](aa) (s2);
    \draw[->] (s3) [out=40,in=0,loop] to coordinate[pos=0.1](aa) (s3);
    \draw[->] (s4) [out=40,in=0,loop] to coordinate[pos=0.1](aa) (s4);
\end{tikzpicture}
        \caption{The agent begins in stage $s=0$, and moves to states $s \in [2n]$, $n = 2$. Different actions correspond to different probability distributions over next states $s \in [2n]$. States $s \in [2n]$ are absording, and rewards are action-independent. Lemma~\ref{lem:informal_single_state_lb} shows that this construction requires the learner to learn $\Omega(n)$ bits about the transition probabilities $p(\cdot \mid 0,a)$.
        \label{fig:mdp_single}}
    \end{figure}

	As described in Figure~\ref{fig:mdp_single}, the hard instances consist of reward-free MDPs that begin in a fixed initial state, and transition to one of $2n$ terminal states according to an unknown transition distribution. The transitions are all taken to be $\epsilon/2n$-close to uniform in the $\ell_{\infty}$ norm, which helps with the embedding later on.  For simplicitiy, the rewards are taken to depend only on states but not on actions. We formalize these instances in the following definition:
	\begin{definition}[Hard Transitions and Rewards at Single State]\label{defn:hard} For parameters $n,A \ge 1$ and $A$, we define the problem class $\Msingle(\epsilon;n,A) : (\Envsingle(n),\Psingle(\epsilon;n,A),\Rsingle(n,A))$ as the triple with the following consitutents:
	\begin{enumerate}
	\item The environment $\Envsingle(n)$ is
	\begin{align*}
	\Envsingle(n,A) = (\Xsingle(n),A,2), \quad \text{where } \Xsingle(n) := \{0,1,\dots,2n\}
	\end{align*}
	\item For a given $\epsilon \in (0,1)$, we define the transition class $\Psingle(\epsilon;n,A)$ as the set of transition operator on $\Envsingle(n,A)$ , parameterized by vectors $q$, which begin at state $x_1 = 0$, and always transition to a state $x_2 \in \{1,\dots,2n\}$ with near-uniform probability, and remain at that state for the remainder of the episode. Formally,
	\begin{align*}
	\Psingle(\epsilon;n,A) &:=  \Big{\{} \Pr[x_1 = 0] = 1, |\Pr[x' = s \mid x = 0, a] - \tfrac{1}{2n}| \le \frac{1}{2n}\epsilon\, \\
	&\qquad\,\Pr[x' = s \mid x = s, a] = 1  \, \, \forall a \in [A], s \in[2n],\,\Big{\}}.
	\end{align*}
	\item We define the hard reward class $\Rsingle(n,A)$ as the set of rewards which as the set of rewards which assign $0$ reward to state $0$, and an action-independent reward to each state $s \in [2n]$. Formally, we define $\Rsingle(n,A) :=  \left\{r_{\nu}:\, r_{\nu}(0,\cdot) = 0,\, r_{\nu}(x,\cdot) = \nu[x], \quad \nu \in [0,1]^{2n}\right\}.$
	\end{enumerate}
	\end{definition}

	\begin{lemma}[Formal Statement of Lemma~\ref{lem:informal_single_state_lb}]\label{lem:single_instance_formal} Fix $\epsilon \le 1$, $p \le 1/2$, $A \ge 2$, and suppose that $n \ge c_0 \log_2 A$ for universal constants $c_0$. Then, there exists a distribution $\calD$ over transition vectors $\Pr \in \Psingle(\epsilon;n,A)$ such that any algorithm which $(\epsilon/12,p)$-learns the class $\Msingle(\epsilon;n,A)$ satisfies
	\begin{align*}
	\Exp_{\Pr \sim \calD}\,\Exp_{\Pr,\Alg}[K] \gtrsim \frac{nA}{\epsilon^2}\,.
	\end{align*}
	\end{lemma}
	Due to its level of technical, the proof of Lemma~\ref{lem:single_instance_formal} is given in Section~\ref{ssec:prf_one_state_lb}.

\subsection{Learning Transitions at $n$ states: Proof of Theorem~\ref{thm:main_lb} \label{ssec:lb_nstates}}
Let $n \ge 2$ be a power of two, which we ultimately will choose to be $\Omega(S)$.  This means that $\lnot:= \log_2 n \in \N$ is integral, and define the layered state space:
	\begin{align*}
	\calX := \left\{(x,\ell): x \in [2^{\ell}], \,\ell \in \{0,1,\dots,\lnot+1\} \right\}
	\end{align*}
	The cardinality of the state space is bounded as $|\calX| \le 1 + 2+ \dots + n/2 + n +  2n \le 4n$. Hence, we shall chose $n$ to be the largest power of two such that $4n \le S$. Note then that $n = \Omega(S)$ as long as $S \ge C$ for a universal constant $C$. We will establish our lower bound for the environment $\Envembed = (\calX, A,H)$, that is, with state space $\calX$; the lower bound extends to an MDP wiht desired state space of size $S$ by augmenting the MDP with isolated, univistable states.

	\paragraph{Description of Transition Class }
	Let us define the class $\Pembed$.
	First, we require that the states $(x,\ell)$ for $\ell \in [\lnot]$ form a dyadic tree, whose transitions are all known to the learner. That is, for $\Pr \in \Pembed$,
	\begin{align*}
	&\Pr[s_1 = (0,1)] = 1\\
	&\Pr[s' = (x,\ell+1) \mid s = (x,\ell), a = 1] = 1, \quad \ell \in \{0,1\dots, \lnot -1\}\\
	&\Pr[x' = (2^{\ell} + x,\ell- 1) \mid s = (x,\ell), a] = 1, \quad \ell \in \{0,1,\dots,\lnot-1\}, \,a > 1.
	\end{align*}
	In words, $\Pr$ starts at $(1,1)$, moves leftward with action $a = 1$, and rightward with actions $a > 1$. At each state $s = (x,\lnot)$, the learn learner faces transitions described by some $\Prsingle^{(x)} \in \Psingle(\epsilon_0)$ for $\epsilon_0= 1/8H$: specifically, we stipulate that states $(x,\lnot)$ always transition to states $(x',\lnot+1)$, which are absorbing:
	\begin{align*}
	&\forall P \in \Pembed, x \in [n], \text{ there exists a } \Prsingle^{(x)} \in \Psingle(\epsilon_0) \text{ such that }:\\
	&\Pr[s' = (x',\lnot+1) \mid s = (x,\lnot), a] = \Prsingle^{(x)}[s' = x' \mid s = 0, a], \,\,\,\forall a \in [A], \,x' \in [2n].\\
	&\Pr[s' = (x',\lnot+1) \mid s = (x',\lnot+1), a] =1, \,\,\, \forall a \in [A]
	\end{align*}
	Thus, there is a bijection between instances $\Pr \in \Pembed$ and tuples $(\Prsingle^{(1)},\dots,\Prsingle^{(n)}) \in \Psingle^n$.



	\paragraph{Description of Reward Class}
	Define the reward class $\Rembed = \{r_{x,\nu}\}$ considering for action-independent rewards
	\begin{align*}
	r_{x,\nu}(s,a) = \begin{cases} 0 & s = (x',\ell),\, \ell < \lnot, \\
	0 & s = (x',\lnot) \text{ and } x' \ne x\\
	1 & s = (x,\lnot)  \\
	r_{\nu}[x'] & s = (x',\lnot + 1).
	\end{cases}
	\end{align*}
	In other words, the learner recieves reward $1$ at state $(x,\lnot)$, rewards $r_{\nu}$ at terminal states $(x',\lnot + 1)$, and $0$ elsewhere. We now establish that any policy which is $\epsilon$-optimal under reward $r_{x,\nu}$ must visit $(y,\lmax)$ with sufficiently high probability:

	\begin{lemma}\label{lem:visit_y} Suppose that a (possibly randomized, non-Markovian) policy $\pi$ satisfies, for $\epsilon \le 1/4$ and $\epsilon_0 \le 1/8H$,
	\begin{align*}
	V^{\pi}(\Pr,r_{x,\nu}) \ge \max_{\pi'} V^{\pi'}(\Pr,r_{x,\nu}) - \epsilon, \quad \forall i \in [N].
	\end{align*}
	Then, $\Pr^{\pi}[s_{\lnot+1} = (x,\lmax)] \ge \frac{1}{2}$.
	\end{lemma}
	\begin{proof}
	Due to the structure of the transitions and rewards, the value of any policy $\pi$ is
	\begin{align*}
	V^{\pi}(\Pr,r_{x,\nu}) = \Pr^{\pi}[s_{\lnot + 1} = (x,\lnot) ] + (H - \lnot - 1) \sum_{x' = 1}^{2n} \nu(x') \Pr^{\pi}[s_{\lnot + 2} = (x,\lnot) ]
	\end{align*}
	Since the transitions from $(x',\lnot)$ to $(x'',\lnot+1)$ is $\epsilon_0/2n$-away from uniform in $\ell_{\infty}$, we can also see that $\Pr^{\pi}[s_{\lnot + 2} = (x,\lnot) ] \in (\frac{1}{2n} - \epsilon, \frac{1}{2n}+\epsilon)$. Thus, letting $\nubar := \frac{1}{2n}\sum_{x'=1}^{2n}\nu[x']$, we have
	\begin{align*}
	\left|(H - \lnot - 1) \sum_{x' = 1}^{2n} \nu(x') \Pr^{\pi}[s_{\lnot + 2} = (x,\lnot) ] - (H-\lnot - 1)\nubar\right| \le (H-\lnot - 1)\epsilon_0 \le \frac{1}{8}.
	\end{align*}
	This entails that
	\begin{align*}
	|V^{\pi}(\Pr,r_{x,\nu}) - (H-\lnot - 1)\nubar -  \Pr^{\pi}[s_{\lnot + 1} = (x,\lnot) ]| \le \frac{1}{8}.
	\end{align*}
	Consequently, by considering a policy $\pi'$ which always visits state $s_{\lnot+1} = (x,\lnot)$ (this can be achieved due to the deterministic behavior of the actions),
	\begin{align*}
	 \max_{\pi'} V^{\pi'}(\Pr,r_{x,\nu}) - V^{\pi}(\Pr,r_{x,\nu}) \ge  1 - \Pr^{\pi}[s_{\lnot + 1} = (x,\lnot) ]  - 2 \cdot \frac{1}{8} = \frac{3}{4} - \Pr^{\pi}[s_{\lnot + 1} = (x,\lnot) ].
	\end{align*}
	In order for the above to be at most $1/4$, we must have that $\Pr^{\pi}[s_{\lnot + 1} = (x,\lnot) ] \ge 1/2$.
	\end{proof}

	\paragraph{Concluding the Proof of Theorem~\ref{thm:main_lb}}

	To prove Theorem~\ref{thm:main_lb}, we use the following lemma:
	\begin{lemma}[Embedding Correspondence]\label{lem:embedding_correspondence} Suppose that $H \ge (2\lnot + 2)$. Then there exists a correspondence $\Psi$, which does not dependent on $\Pr \in \Pembed$ or $r_{y,\nu} \in \Rembed$ (but possibly on $\epsilon,n,A,H$) which operates as follows: Given a policy $\pi$ for $\Eembed$, $\Psi[\pi] = (\pi^{(1)},\dots,\pi^{(n)})$ returns an $n$-tuple of policies for $\Envsingle(n,A)$ with the following property: For any $\Pr \equiv (\Prsingle^{(1)},\dots,\Prsingle^{(n)}) \in \Pembed$ and $r_{x,\nu} \in \Rembed$,
	\begin{align*}
	\text{If }V^{\pi}(\Pr,r_{x,\nu}) \ge \max_{\pi'} V^{\pi'}(\Pr,r_{x,\nu}) - \epsilon, \quad  \forall x \in [n], \quad V^{\pi^{(x)}}(\Prsingle^{(x)},r_{\nu}) \ge \max_{\pi'} V^{\pi'}(\Prsingle^{(x)},r_{\nu}).
	\end{align*}
	\end{lemma}
	\begin{proof}[Proof of Lemma~\ref{lem:embedding_correspondence}] We directly construct the map $\Psi$. Observe that policies $\pi^{(x)}$ on the single state environment can be discred by a distribution over which actions $a \in [A]$ they select at the initial state $x$. Thus identifying policies as elements of $\Delta(A)$, we set
	\begin{align*}
	\pi^{(x)}[a] :=  \begin{cases} \Pr^{\pi}[ a_{\lnot + 1} = a \mid s_{\lnot + 1} = (x,\lnot)] & \Pr^{\pi}[s_{\lnot + 1} = (x,\lnot )] > 0  \\
	\text{arbitrary} & \text{otherwise}
	\end{cases}
	\end{align*}
	as the marginal distribution of actions selected when $s_{\lnot+1} = (x,\lnot + 1)$. Observe that the above conditional probabilites \emph{do not} depend on $\Pr \in \Pembed$ since the dynamics up to $h = \lnot + 1$ are identical for all instances. By considing a policy which coincides with $\pi$ until $s_{\lnot + 1} = (x,\lnot)$ and swtiches to playing optimally, we can lower bound the subopitmality of $\pi$ by 
	\begin{multline*}
	\max_{\pi'} V^{\pi'}(\Pr,r_{x,\nu})- V^{\pi}(\Pr,r_{x,\nu}) \ge \\
	\Pr^{\pi}[s_{\lnot + 1} = (x,\lnot)] \cdot (H - \lnot -1)\left(\max_{\pi'}V^{\pi}(\Prsingle^{(x)},r_{\nu})-V^{\pi^{(x)}}(\Prsingle^{(x)},r_{\nu})\right)
	\end{multline*}
	 In particular, if $\pi$ is $\epsilon \le 1/4$-suboptimal, then Lemma~\ref{lem:visit_y} ensures $\Pr^{\pi}[s_{\lnot + 1} = (x,\lnot )] \ge 1/2$. Since $H \ge 2(\lnot+1)$ by assumption, we have
	\begin{align*}
	\epsilon \ge \max_{\pi'} V^{\calM,\pi'}- V^{\calM,\pi} \ge \frac{H}{4}\left(\max_{\pi'}V^{\pi}(\Prsingle^{(x)},r_{\nu})-V^{\pi^{(x)}}(\Prsingle^{(x)},r_{\nu})\right),
	\end{align*}
	Therefore,  $\max_{\pi'}V^{\pi}(\Prsingle^{(x)},r_{\nu})-V^{\pi^{(x)}}(\Prsingle^{(x)},r_{\nu}) \le \frac{4\epsilon}{H}$, as needed. \end{proof}

	We now conclude with the proof of our main theorem:
	\begin{proof}[Proof of Theorem~\ref{thm:main_lb}.]
	\newcommand{\Algsingle}{\Alg_{\mathrm{single}}}

	 Let $\Alg$ be $(\epsilon,p)$-correct on the class $(\Envembed,\Pembed,\Rembed)$.  Then, for any $x \in [2n]$, we simulate obtain a $(4\epsilon/H,p)$-correct algorithm for $\Msingle(4\epsilon/H;n,A)$ as follows:
	\begin{enumerate}
	\item Exploration: Let $\calD$ be the distribution over $\Prsingle \in \Psingle$ from Lemma~\ref{lem:single_instance_formal}. Draw a tuple $\Pr^{\ne x} = (\Prsingle^{(x')})_{x' \ne x}$ of $n-1$ distributions i.i.d from $\calD$, and let $\Algsingle^{(x,\Pr^{\ne x})}$ denote the algorithm induced by embeding the instance in $\Msingle(4\epsilon/H;n,A)$ at stage $x$ of the embedding construction, running $\Alg$ on this embedded instance
	\item Planning: When queried given a reward vector $r_{\nu} \in \Rsingle$, use $\Alg$ to compute a policy $\pi$ for reward vector $r_{x,\nu} \in \Rembed$, and return the policy $\pi^{(x)}$ dicated by the corresponding $\psi$.
	\end{enumerate}
	Since $\Alg$ is $(\epsilon,p)$-correct and $\epsilon \le 1/4$, the correspondence $\Psi$ ensures that for any draw of $\Pr^{\ne x}$, $\Algsingle^{(x,\Pr^{\ne x})}$ is $(4\epsilon/H,p)$-correct. Let $K^{(x,\Pr^{\ne x})}$ denote the random number of episodes collected by $\Algsingle^{(x,\Pr^{\ne x})}$ in the exploration phase,  Thus, if $\epsilon \le \min\{\frac{1}{4},\frac{H}{48}\}$, and $n \ge c_0\log_2 A$ for the appropriate $c_0$ specified in Lemma~\ref{lem:single_instance_formal}, the Lemma~\ref{lem:single_instance_formal} entails
	\begin{align*}
	\Exp_{\Prsingle \sim \calD} \Exp_{\Prsingle,\Algsingle^{(x,\Pr^{\ne x})}}[K^{(x,\Pr^{\ne x})}] \gtrsim \frac{nAH^2}{\epsilon^2}.
	\end{align*}
	By taking an expectation over $\Pr^{\ne x}$, we have
	\begin{align*}
	\Exp_{\Pr^{\ne x}\sim \calD^{n-1},\Prsingle \sim \calD} \Exp_{\Prsingle,\Algsingle^{(x,\Pr^{\ne x})}}[K^{(x,\Pr^{\ne x})}] \gtrsim \frac{nAH^2}{\epsilon^2}.
	\end{align*} Note then that, if $N_K(x)$ denotes the number of times that the original $\Alg$ visits state $(x,\lnot)$, then, by Fubini's theorem and the contruction of $\Algsingle^{(x,\Pr^{\ne x})}$, the expectation of $N_K(x)$ under probabilities drawn uniform from $\calD^n$ is euqal to the expectation of $K^{(x,\Pr^{\ne x})}$ where $\Pr^{\ne x}$ is drawm uniformly from $\calD^{n-1}$, and then the transition $\Prsingle$ is selected. Formally,
	\begin{align*}
	\Exp_{\Pr^{\ne x}\sim \calD^{n-1},\Prsingle \sim \calD} \Exp_{\Prsingle,\Algsingle^{(x,\Pr^{\ne x})}}[K^{(x,\Pr^{\ne x})}] = \Exp_{\Pr \equiv (\Prsingle^{(1)},\dots, \Prsingle^{(n)}) \sim \calD^{n}} \Exp_{\Pr,\Alg}[K_x]
	\end{align*}
	This implies that
	\begin{align*}
	\Exp_{\Pr = (\Prsingle^{(1)},\dots, \Prsingle^{(n)}) \sim \calD^{n}} \Exp_{\Pr,\Alg}[K_x] \gtrsim \frac{nAH^2}{\epsilon^2}.
	\end{align*}
	Since the number of episodes $K$ encounted by $\Alg$ is equal to $\sum_{x=1}^n K_x$ (the agent visits exactly one state of the form $(x,\lnot)$ per episode), we have
	\begin{align*}
	\Exp_{\Pr = (\Prsingle^{(1)},\dots, \Prsingle^{(n)}) \sim \calD^{n}} \Exp_{\Pr,\Alg}[K] \gtrsim \sum_{x=1}^n \frac{nAH^2}{\epsilon^2} = \frac{n^2AH^2}{\epsilon^2}.
	\end{align*}

	Since $S/8 \le n \le S$, for the above conditions to hold, it suffices that, for a sufficiently large constant $C$, $S \ge C \log_2 A$, $\epsilon \le \min\{\frac{1}{4},\frac{H}{48}\}$, and $H \ge C \log_2 S$. Moreover, $\frac{n^2AH^2}{\epsilon^2} = \Omega(\frac{S^2AH^2}{\epsilon^2})$, as needed.
	\end{proof}

\subsection{Proof of Lemma~\ref{lem:single_instance_formal}\label{ssec:prf_one_state_lb}}
	\paragraph{A packing of reward-free MDPs} The first step is to construct a family of transition probabilities $\Pr_J \in \Pclass(\epsilon;n,A)$ which witness the lower bound. Let $\vecone$ denote the all ones vector on $[2n]$. To construct the packing, we define the set of binary vectors
	\begin{align*}
	\calK := \left\{v \in \{-1,1\}^{2n}~: \vecone^\top v = 0\right\}.
	\end{align*}
	\newcommand{\calV}{\mathcal{V}}
	For a cardinality parameter $M$ to be chosen shortly, we consider a packing of vectors
	\begin{align*}
	\calV_{A,M} := \{v_{a,j} \in \calK: a \in [A],j \in [M]\}
	\end{align*}
	Throughout, we shall consider packings $\calV_{A,M}$ which are \emph{uncorrelated} in the following sense:
	\begin{definition}[Uncorrelated] For $\gamma \in (0,1)$, we say that $\calV_{A,M}$ is $\gamma$-uncorrelated if, for any pair $(a,j),(a',j')$ with \emph{either} $a \ne a'$ or $j \ne j'$, it holds that $|\langle v_{a,j}, v_{a',j'} \rangle| < 2n \gamma.
	$.
	\end{definition}
	The following lemma shows that the exist $\gamma$-uncorrelated packings of size $e^{\Omega(n\gamma^2)}$:
	\begin{lemma}\label{lem:v_uncor} Fix $\gamma \in (0,1)$, and suppose that $2\log(M) \le n\gamma^2 - \log(4n) - 2\log(A)$. Then, there exists a $\gamma$-uncorrelated packing $\calV_{A,M}$.
	\end{lemma}
	\begin{proof}[Proof Sketch] We use the probabilistic method. Specifically, we draw $v_{a,j} \unifsim \calK$, and can bound $\langle v_{a,j},v_{a',j'} \rangle$ with high-probability Chernoff bounds. Taking a union bound shows that an uncorrelated packings arise from this construction with non-zero probability. A full proof is given in in Section~\ref{sssec:lem:v_uncor}.
	\end{proof}

	Given a $\gamma$-uncorrelated packing $\calV_{A,M}$,  define transition vectors
	\begin{align*}
	q_{a,j} :=  q_0 + \frac{\epsilon}{2n}v_{a,j}, \text{ where } q_0 = \frac{1}{2n}\vecone.
	\end{align*}
	Since $\epsilon \le 1$ and $\vecone^\top v_{a,j_a} = 0$, $q_{j,a}\in \Delta(2n)$. Wet indices $J$ denote tuples $J = (J_1,\dots,J_A) \in [M]^A$, let $q_J(\cdot,a) = q_{a,J_a}$, and define $\Pr_J$ as the instance $\Pr_{q_J}$, where $\Pr_q$ is as in Definition []. Formally,
	\begin{align*}
	\Pr_J : \quad \Pr^{\Pr_J}[s_1 = 0] = 1,\,\Pr^{\Pr_J}[s_2 = 0] = 0,\,\, \forall s \in [2n],\, \Pr^{\Pr_J}[s_2 = s \mid s_1 = 0, a] = q_{J}(s,a) = q_{a,J_a}(s)
	\end{align*}
	\paragraph{Lower Bound for Estimating the Packing Instance: } Let us suppose we have an exploration algorithm $\Algest$ which, for any $\Pr_J$, collects (a possibly random number) $K$ trajectories, and returns estimates $\Jhat_1,\dots,\Jhat_A$ of $J_1,\dots,J_A$. Our first step is to establish a lower bound on $K$ assuming that $\Algest$ satisfies a uniform correctness guarantee:
	\begin{lemma}\label{lem:lb_fano_algest} For any $\Algest$ satisfying the guarantee
	\begin{align}
	\forall J \in [A]^M,\,\, \Pr_{\Pr_J,\Algest}\left[\Jhat_a = J_a\,\forall a \in [A]\right] \ge 1 - a. \label{eq:algest_guarantee}
	\end{align}
	Then, we must have
	\begin{align*}
	\Exp_{J \unifsim [A]^M} \Exp_{\Pr_J,\Algest}[K] \ge A \cdot \frac{(1 - p)\log M - \log 2}{\epsilon^2}
	\end{align*}
	\end{lemma}
	The above bound essentially follows from an application of Fano's inequality, and is proven in Section~\ref{sssec:lem_lb_fano_algest}. In particular, if we take say $p = 1/2$, and require $M = e^{\Omega(S)}$, then we have  $\Exp_{J \unifsim [A]^M} \Exp^{\Pr_J,\Algest}[K] \gtrsim \frac{SA}{\epsilon^2},$ as desired.

	\paragraph{Estimation Reduces to Exploration} Of course, the above bound applies only to an estimation algorithm $\Algest$, but our intent is to establish lower bounds for exploration algorithms.  In the following lemma, we state that if the packing is suffciently uncorrelated, then we can convert an $(\epsilon/24,p)$-correct exploration algorithm into an Algorithm $\Algest$ satisfying Eq.~\eqref{eq:algest_guarantee}.
	\begin{lemma}\label{lem:estimation_reduce_exploration} Suppose $\Alg$ is $(\epsilon/24,p)$-correct on the class $\Msingle(\epsilon,n,A)$, and that the packing $\calV_{M,A}$ is $\gamma = 1/10$-uncorrelated. Then, there is an algorithm $\Algest$ which collects $K$ trajectories according to $\Alg$, and satisfies Eq.~\ref{eq:algest_guarantee}.
	\end{lemma}
	\begin{proof}[Proof Sketch] Consider reward vectors $r_{\nu}$ induced by $\nu_{a,j,a_2,j_2} \propto 2q_{a,j} - q_{a_2,j_2}$. These reward vectors can be used to ``pick out'' $q_{a,J_a}$ as follows.
	For a given $a$, we show that on the good exploration event, $\Alg$ returns policies with $\Pr[\pihat_1^{\nu}(0) = a] > 1/2$ for all $\nu = \nu_{a,J_a,a_2,j_2}$ ranging across $a_2,j_2$. However, for $j \ne J_a$, we show that on this good event there exists some $a_2,j_2$ for which $\Alg$ returns policies with $\Pr[\pihat_1^{\nu}(0)  = a] < 1/2$. Hence, we can estimate $q_{a,J_a}$ by finding the (say, the first) index $j$ for which $\Pr[\pihat_1^{\nu}(0) = a] > 1/2$ for all $\nu = \nu_{a,j,a_2,j_2}$, ranging across $a_2,j_2$. A full proof is given in Section~\ref{sssec:lem:estimation_reduce_exploration}.
	\end{proof}
	As a consequence, we find that if $\gamma \le 1/10$ and $\Alg$  is $(\epsilon/24,p)$-correct,
	\begin{align*}
	\Exp_{J \unifsim [A]^M} \Exp_{\Pr_J,\Alg}[K] \ge A \cdot \frac{(1 - p)\log M - \log 2}{\epsilon^2}
	\end{align*}
	In particular, if $\log M \ge 4 \log 2$ and $p \le 1/2$, then,
	\begin{align}
	\Exp_{J \unifsim [A]^M} \Exp_{\Pr_J,\Alg}[K] \ge A \cdot \frac{\log M}{4\epsilon^2} \label{eq:lower_bound_stuff}
	\end{align}

	\paragraph{Concluding the proof} Take $\gamma = 1/10$. For constants $c_0,c_1$ sufficiently large, we can ensure that if $n \ge c_0 \log_2 A$, then $M = e^{-n/c_1}$ statisfies $2\log(M) \le n\gamma^2 - \log(4n) - 2\log(A)$ and $\log M \ge 4 \log 2$. Thus, we can construct a $\gamma$-uncorrelated packing of cardinality $ \log M \ge n/c_1$,
	\begin{align*}
	\Exp_{J \unifsim [A]^M} \Exp_{\Pr_J,\Alg}[K] \ge A \cdot \frac{n}{4c_1\epsilon^2} ,
	\end{align*}
	as needed. \qed

	\subsubsection{Proof of Lemma~\ref{lem:v_uncor}\label{sssec:lem:v_uncor}}

		We begin with the following concentration inequality:
		\begin{lemma}\label{lem:v_concentration} For any fixed $(a,j)$ and $(a',j')$, we have
		\begin{align*}
		\Pr[|\langle v_{a,j} , v_{a',j'}\rangle| \ge 2n\gamma] \le e^{ \log (4n) - n\gamma^2}.
		\end{align*}
		\end{lemma}
		\begin{proof} By permuting coordinates, we may assume that
		\begin{align*}
		v_{a',j'}[s] = \begin{cases} 1 & s \in [n]\\
		-1 & s \in \{n+1,\dots,2n\}\end{cases}\,.
		\end{align*}
		Then,
		\begin{align*}
		\langle v_{a,j} , v_{a',j'}\rangle &= 2|\{s\in [n]: v_{a,j}[s] = 1\}| - 2(n - |\{s\in [n]: v_{a,j}[s] = 1\}|) \\
		&= 2n - 4|\{s\in [n]: v_{a,j}[s] = 1\}| := 2n - 4Z,
		\end{align*}
		where we set $Z = |\{s\in [n]: v_{a,j}[s] = 1\}|$. Hence, if $|\langle v_{a,j} , v_{a',j'}\rangle|  \ge 2\gamma n$, we need
		\begin{align*}
		\left|\frac{Z}{n} - \frac{1}{2}\right| \ge \frac{\gamma}{2}.
		\end{align*}
		Now, we have that for $i \in [n]$,
		\begin{align*}
		\Pr[Z = i] < \frac{\binom{n}{i} \cdot \binom{n}{n-i}}{\sum_{i=0}^n \binom{n}{i} \cdot \binom{n}{n-i}} = \frac{\binom{n}{i}^2}{\sum_{i=0}^n \binom{n}{i}^2} < n\frac{\binom{n}{i}^2}{\left(\sum_{i=0}^n \binom{n}{i}\right)^2} = n\Pr_{W \sim \mathrm{Binom}(n,1/2)}[W=i]^2.
		\end{align*}
		Hence,
		\begin{align*}
		\Pr\left[\left|\frac{Z}{n} - \frac{1}{2}\right| \ge \frac{\gamma}{2}\right] &\le n\sum_{i:|\frac{i}{n} - \frac{1}{2}| \ge \frac{\gamma}{2}} \Pr_{W \sim \mathrm{Binom}(n,1/2)}[W=i]^2\\
		&\le n\left(\sum_{i:|\frac{i}{n} - \frac{1}{2}| \ge \frac{\gamma}{2}} \Pr_{W \sim \mathrm{Binom}(n,1/2)}[W=i]\right)^2\\
		&= n\left(\Pr_{W \sim \mathrm{Binom}(n,1/2)}\left[\left|\frac{W}{n} - \frac{1}{2}\right| \ge \frac{\gamma}{2}\right]\right)^2
		&\le n(2e^{- 2(\gamma/2)^2 n})^2 = e^{ \log (4n) - n\gamma^2}
		\end{align*}
		\end{proof}

		We now finish the proof of our intended lemma:
		\begin{proof}[Proof of Lemma~\ref{lem:v_uncor}]
		 By a union bound over at most $A^2M^2 - 1$ pairs $(a,j),(a',j')$, there exists a $\gamma$-uncorrelated packing for any $M$ satisfying
		\begin{align*}
		A^2M^2e^{ \log (4n) - n\gamma^2} \le 1
		\end{align*}
		Taking logarithms, we require $2\log(M) \le n\gamma^2 - \log(4n) - 2\log(A)$.

		\end{proof}

	\subsubsection{Proof of Lemma~\ref{lem:lb_fano_algest}\label{sssec:lem_lb_fano_algest}}

		To begin, let us state a variant of Fano's inequality, which replaces mutual-information with an arbitrary comparison measure:
		\begin{lemma}[ Fano's Inequality ] Consider $M$ probability measures $\Pr_{1},\dots,\Pr_{M}$ on a space $\Omega$. Then for any estimator $\jhat$ on $\Omega$ and any comparison law $\Pr_0$ on $\Omega$,
		\begin{align*}
		\frac{1}{M}\sum_{j=1}^M \Pr_{j}\left[\jhat \ne j\right] \ge 1 - \frac{\log 2 + \frac{1}{M}\sum_{j=1}^M\KL(\Pr_j,\Pr_0) }{\log M}
		\end{align*}
		\end{lemma}
		\begin{proof} This follows from the standard statement of Fano's inequality, where we use that
		\begin{align*}
		\inf_{\Pr_0}\frac{1}{M}\sum_{j=1}^M\KL(\Pr_j,\Pr_0) =  \frac{1}{M}\sum_{j=1}^M\KL\left(\Pr_j,\frac{1}{M}\sum_{j'=1}^M\Pr_{j'}\right)
		\end{align*}
		For reference, see e.g. Equation (11) in~\cite{chen2016bayes}.
		\end{proof}
		We will apply Fano's inequality of each $a \in [A]$. To begin, for a fixed $J \in [M]^A$ and $a \in [A]$, let us define the laws ``$\Pr_{j}$''. We let $\Pr_{J,a,j}$ denote the reward-free MDP with starting at $x = 0$ deterministically, and with transitions
		\begin{align*}
		\Pr^{\Pr_{J,a,j}}[s \mid x_1 = 0,a_1 = a'] = \begin{cases} q_{a,j}[s] & a' = a\\
		q_{a',J_{a'}}[s] & a' \ne a.
		\end{cases}
		\end{align*}
		For fixed $J,a$, we let $\Pr_{j;J,a}$ denote the joint law induced by $\Algest$ and $\Pr_{J,a,j}$. For the comparison measure, let $\Pr_{J,a,0}$ denote the analogous MDP to $\Pr_{J,a,j}$, but where $\Pr^{\Pr_{J,a,j}}[s \mid x_1 = 0,a_1 = a] = q_{0}$ for the fixed action $a$. We let $\Pr_{0;J,a}$ denote the law induced by $\Algest$ and $\Pr_{J,a,j}$. Then, Fano's iqequality implies that
		\begin{align}\label{eq:Fano_conclusion}
		\forall J,a, \quad (1 - p)\log M - \log 2 \le  \frac{1}{M}\sum_{j=1}^M \KL(\Pr_{J,a,j},\Pr_{0;J,a}).
		\end{align}
		Now, observe that the laws $\Pr_{J,a,j}$ and $\Pr_{0;J,a}$ only differ due to transitions selecting action $a_1 = a$. Under the first law, these have distribution $\mathrm{Multinomial}(q_{a,j})$, and under the second, $\mathrm{Multinomial}(q_{0})$. Let $N_K(a = a_1)$ denote the expected number of times algorithm $\Algest$ selects action $a_1 = a$ at time step $1$. From a Wald's identity argument (see e.g. \cite{kaufmann2016complexity}), we have
		\begin{align*}
		\KL(\Pr_{J,a,j},\Pr_{0;J,a}) &= \Exp_{\Pr_{J,a,j},\Algest}[N_K(a_1 = a)] \,\KL(\mathrm{Multinomial}(q_{a,j}),\mathrm{Multinomial}(q_{a,0} ))\\
		&= \Exp_{\Pr_{J,a,j},\Algest}[N_K(a_1 = a)] \,\sum_{s=1}^{2n} \frac{1 + \epsilon v_{j,a}[s]}{2n}\log(1 + \epsilon v_{j,a}[s])\\
		&\overset{(i)}{\le} \Exp_{\Pr_{J,a,j},\Algest}[N_K(a_1 = a)] \,\sum_{s=1}^{2n} \frac{\epsilon v_{j,a} + \epsilon^2 v_{j,a}[s]^2}{2n}\\
		&\overset{(ii)}{\le} \epsilon^2 \cdot \Exp_{\Pr{J,a,j},\Algest}[N_K(a_1 = a)]
		\end{align*}
		where $(i)$ uses $1 + \epsilon v_{j,a}[s] \ge 0$ and the identity $\log (1+ x) \le x$, and $(ii)$ uses the fact that $v_{j,a}[s]^2 = 1$ and $\sum_{s = 1}^{2n} v_{j,a}[s] = 0$ for $v_{j,a} \in \calK$.  Thus, by Eq~\ref{eq:Fano_conclusion},
		\begin{align*}
		\forall J,a, \quad \frac{(1 - p)\log M - \log 2}{\epsilon^2} \le  \frac{1}{M}\sum_{j=1}^M \Exp_{\Pr_{J,a,j},\Algest}[N_K(a_1 = a)].
		\end{align*}
		By taking an expectation over index tuples $J$ drawn uniformly from $[A]^M$, we have
		\begin{align*}
		\forall a, \quad \frac{(1 - p)\log M - \log 2}{\epsilon^2} &\le  \frac{1}{M}\sum_{j=1}^M \Exp_{J \unifsim [A]^{M}}\Exp_{\Pr_{J,a,j},\Algest}\left[N_K(a_1 = a)\right]  \\
		&= \Exp_{J \unifsim [A]^{M}}\Exp_{\Pr_J,\Algest}\left[N_K(a_1 = a)\right],
		\end{align*}
		where the last line follows that $\Pr_{J,a,j} = \Pr_{J'}$ for some $J'$ and that, by symmetry, each index $J'$ has equal weight when averaged over both $J \in [A]^M$ and $j \in [M]$. Summing over $a \in [A]$, we have
		\begin{align*}
		A \cdot \frac{(1 - p)\log M - \log 2}{\epsilon^2} \le \Exp_{J \unifsim [A]^{M}}\Exp_{\Pr_J,\Algest}\left[\sum_{a=1}^A N_K(a_1 = a)\right] = \Exp_{J \unifsim [A]^{M}}\Exp_{\Pr_J,\Algest}[K].
		\end{align*}

	\subsubsection{Proof of Lemma~\ref{lem:estimation_reduce_exploration} \label{sssec:lem:estimation_reduce_exploration}}
		Let us now show that $(\epsilon/12,p)$-learning implies the existence of an algorithm $\Algest$ satisfying Eq.~\ref{eq:algest_guarantee}, provided the packing is sufficiently uncorrelated. Introduce the vectors
		\begin{align*}
		\nu_{a_1,a_2,j_1,j_2} := \frac{1}{3}v_{a_1,j_1} + \frac{1}{6}v_{a_2,j_2} + \frac{1}{2}\vecone,
		\end{align*}
		which can be checked to lie $[0,1]^{2n}$. We shall establish the following lemma, which says that for sufficciently uncorrelated packings, the vectors $\nu_{(\dots)}$ witness separations between $q_{a_1,j_1}$ and $q_{a_2,j_2}$ for different actions $a_1,a_2$:
		\begin{lemma}\label{lemma:inner_product_lem} Fix $a_1 \in [A]$ and $j_1 \in [M]$, and suppose the packing is $\gamma = 1/10$-uncorrelated: Then, for any $a_2 \ne a_1$ and $j_2 \in [M]$, the following holds
		\begin{align*}
		&\min_{a_2',j_2'}\langle q_{a_1,j_1} - q_{a_2,j_2}, \nu_{a_1,a_2',j_1,j_2'} \rangle >  \frac{\epsilon}{12}\\
		\forall j_1' \ne j_1,\,\, & \min_{a_2',j_2'}\langle q_{a_1,j_1} - q_{a_2,j_2}, \nu_{a_1,a_2',j_1',j_2'} \rangle <  - \frac{\epsilon}{12}
		\end{align*}
		\end{lemma}
		\begin{proof}[Proof of Lemma~\ref{lemma:inner_product_lem}]
		\begin{align*}
		\langle q_{a_1,j_1} - q_{a_2,j_2}, \nu_{a_1',a_2',j_1',j_2'} \rangle &= \frac{\epsilon}{2n}\langle v_{a_1,j_1} - v_{a_2,j_2}, \nu_{a_1',a_2',j_1',j_2'}\rangle\\
		&= \frac{\epsilon}{12n} \langle v_{a_1,j_1} - v_{a_2,j_2}, 2v_{a_1',j_1'} - v_{a_2',j_2'} \rangle,
		\end{align*}
		where we use the fact that $v_{a,j}^\top \vecone = 1$ for all $a,j$. If $a'_1 = a_1$ and $j_1' = j_1$, and the packing is $\gamma \le 1/6$-uncorrelated
		\begin{align*}
		\langle q_{a_1,j_1} - q_{a_2,j_2}, \nu_{a_1,a_2',j_1,j_2'} \rangle  &=  \frac{\epsilon}{12n} \langle v_{a_1,j_1} - v_{a_2,j_2}, 2v_{a_1,j_1} - v_{a_2',j_2'} \rangle\\
		 &= \frac{\epsilon}{12n}\left( 2\langle v_{a_1,j_1},v_{a_1,j_1}\rangle  - 2\langle v_{a_2,j_2},v_{a_1,j_1}\rangle + \langle v_{a_1,j_1}, v_{a_2',j_2'} \rangle -  \langle v_{a_2,j_2}, v_{a_2',j_2'} \rangle\right)\\
		&> \frac{\epsilon}{12n}\left( 4n - 4\gamma n - 2n - 2n \gamma \right)\\
		&\ge \frac{\epsilon}{12n}\left( 2n - 6n\gamma \right) = \frac{\epsilon}{12}.
		\end{align*}
		On the other hand, if $j_1 \ne j_1'$, but $(a_2,j_2) = (a_2',j_2')$ then a similar computation reveals that for $\gamma \le 1/10$,
		\begin{align*}
		\langle q_{a_1,j_1} - q_{a_2,j_2}, \nu_{a_1,a_2,j_1',j_2} \rangle < \frac{\epsilon}{12n}\left( 10\gamma n   - 2n\rangle\right) < \frac{-\epsilon}{12}.
		\end{align*}
		\end{proof}

		We can now conclude the proof of our reduction:
		\begin{proof}[Proof of Lemma~\ref{lem:estimation_reduce_exploration}] Suppose that $\Alg$ is run on $\Pr_J$ for $J \in [M]^A$. Further, recall the rewards $r_{\nu}$ which assign reward of $r_{\nu}(s,a) = \I(s \in [2n]) \nu(s)$. By $(\epsilon/24,p)$-correctness of $\Alg$, then with probability $1-p$, $\Alg$ computes policies $\pihat_{\nu}$ which satisfies the following bound simultaneously for all $\nu \in \{\nu_{a_1,a_2,j_1,j_2}\}$:
		\begin{align}
		\max_{\pi} V^{\pi}(\Pr_J,r_{\nu})  - V^{\pihat_{\nu}}(\Pr_J,r_{\nu})
		 \le \epsilon/24. \label{eq:good_event_lb}
		\end{align}
		For a possibly randomized policy, we use the shorthand $\pi[a]$ to denote the probability of selecting $a$ at the initial state $0$; that is $\Pr^{\pi}[a_1 = a]$. Now, Consider the following procedure: for each $a \in [A]$, estimate $J_a$ by returning the first $j \in [M]$ for which
		\begin{align}
		\forall a_2',j_2', \quad \pihat_{\nu_{a,a_2',j,j_2'}}[a] > 1/2.\label{eq:id_condition}
		\end{align}
		We conclude our proof by showing that, on the good event Eq.~\eqref{eq:good_event_lb}, the condition in Eq.~\eqref{eq:id_condition} holds if and only if $j = J_a$. To this end, define the short hand
		\begin{align*}
		q_{\pi} := \sum_{a'} \pi[a'] q_{a',J_{a'}}
		\end{align*}
		Then, we have that
		\begin{align*}
		\max_{\pi} V^{\pi}(\Pr_J,r_{\nu})  - V^{\pihat_{\nu}}(\Pr_J,r_{\nu}) =  \max_{\pi}\left\langle q_{\pi} - q_{\pihat_{\nu}} , \,\nu\right\rangle,
		\end{align*}
		so that on the good event of Eq.~\ref{eq:good_event_lb}, we have
		\begin{align*}
		 \max_{\pi}\left\langle q_{\pi} - q_{\pihat_{\nu}} , \,\nu\right\rangle \le \frac{\epsilon}{24}.
		\end{align*}

		\paragraph{True Positive for $j = J_a$:} First let's show that Equation~\ref{eq:id_condition} holds for $j = J_a$. Indeed, if it does not, then there exists some $a_2',j_2'$ for which $\Pr[\pihat_{\nu_{a,j,a_2',j_2'}}[a] ] \le 1/2$, and (setting $\nu = \nu_{a,j,a_2',j_2'}$ for shorthand in $\pihat^{\nu}$)
		\begin{align*}
		\epsilon/24 &\ge  \max_{\pi}\left\langle q_{\pi} - q_{\pihat_{\nu}} , \,\nu\right\rangle,\\
		&\ge \left\langle q_{a,J_{a}} - q_{\pihat^{\nu}} , \,\nu_{a,j,a_2',j_2'}\right\rangle \tag*{(choose $\pi[a] =1$)}\\
		&= \sum_{a' \ne a}\pihat_{\nu}[a'] \left\langle q_{a,J_{a}} - q_{a',J_{a'}} , \,\nu_{a,j,a_2',j_2'}\right\rangle\\
		&\ge\underbrace{(1-\pihat_{\nu}[a]) }_{\ge 1/2} \cdot \underbrace{\min_{a' \ne a}\left\langle q_{a,J_{a}} - q_{a',J_{a'}} , \,\nu_{a,j,a_2',j_2'}\right\rangle}_{> \epsilon/12 \text{ by Lemma~\ref{lemma:inner_product_lem} }} > \frac{\epsilon}{24},
		\end{align*}
		yielding a contradiction.

		\paragraph{True Negative for $j \ne J_a$:} On the other hand, for $j \ne J_{a}$ suppose that for all all $a_2' \ne a$ and all $j_2' \in [M]$,  $\Pr[\pihat_1^{\nu_{a,j,a_2',j_2'}}(0) = a] > 1/2$. Then, considering $a_2' = a_2$ and $j_2' = J_{a_2}$, we have  (setting $\nu = \nu_{a,j,a_2,J_{a_2}}$ for shorthand in $\pihat^{\nu}$)
		\begin{align*}
		\epsilon/24 &\ge  \max_{a'}\left\langle q_{a',J_{a'}} - q_{\pihat^{\nu}} , \,\nu_{a,j,a_2,J_2}\right\rangle\\
		&\ge  \left\langle q_{a_2,J_{a_2}} - q_{\pihat^{\nu}} , \,\nu_{a,j,a_2,J_2}\right\rangle\\
		&\ge  \underbrace{\pihat_{\nu}[a_2]}_{\ge \pihat_{\nu}[a] > 1/2} \cdot\underbrace{\min_{a' \ne a_2}\left\langle q_{a_2,J_{a_2}} - q_{a',J_{a'}} , \,\nu_{a,j,a_2',j_2'}\right\rangle}_{> \epsilon/12 \text{ by Lemma~\ref{lemma:inner_product_lem} }} > \frac{\epsilon}{24},
		\end{align*}
		again drawing a contradiction.
		\end{proof}

\end{document}